\documentclass{article}

\PassOptionsToPackage{authoryear}{natbib}
\usepackage[preprint]{neurips_2021}

\usepackage[utf8]{inputenc} % allow utf-8 input
\usepackage[T1]{fontenc}    % use 8-bit T1 fonts
\usepackage[parfill]{parskip}
\usepackage{amsmath,amsthm,amssymb,bm,bbm}
\usepackage{mathtools}
\usepackage{cases}

\usepackage{microtype}

\usepackage{algorithm,algorithmic}
\usepackage{color}
\usepackage{appendix}

\usepackage{url}
\usepackage[colorlinks,citecolor=blue,urlcolor=blue,linkcolor=blue,linktocpage=true]{hyperref}
\usepackage{etoolbox}

\usepackage{cleveref}
\crefformat{equation}{(#2#1#3)}
\crefrangeformat{equation}{(#3#1#4) to~(#5#2#6)}
\crefname{equation}{}{}
\Crefname{equation}{}{}

\crefname{definition}{\textbf{definition}}{definitions}
\Crefname{definition}{Definition}{Definitions}
\crefname{assumption}{\textbf{assumption}}{assumptions}
\Crefname{assumption}{Assumption}{Assumptions}

\definecolor{maroon}{RGB}{192,80,77}

\newcommand{\explain}[2]{\underset{\mathclap{\overset{\uparrow}{#2}}}{#1}}
\newcommand{\explainup}[2]{\overset{\mathclap{\underset{\downarrow}{#2}}}{#1}}

\newtheorem{theorem}{Theorem}
\newtheorem{lemma}[theorem]{Lemma}

\newtheorem{definition}[theorem]{Definition}

\newtheorem{assumption}{Assumption}
\newcommand{\argmax}{\mathop{\mathrm{argmax}}}

\def\E{\mathbb{E}}
\def\P{\mathbb{P}}

\def\R{\mathbb{R}}
\def\cA{\mathcal{A}}

\def\cN{\mathcal{N}}

\usepackage{etoolbox}
\usepackage{comment}
\usepackage{caption}
\usepackage{subcaption}
\usepackage{isomath}
\usepackage{indentfirst}
\usepackage{enumerate}
\usepackage{enumitem}
\setlist{leftmargin=10mm}
\usepackage{mathrsfs}
\usepackage{multirow}
\usepackage{amstext}
\def\ind{\mathbbm{1}}

\newbool{twocol}
\setbool{twocol}{false}

\newbool{compact}  % set to true to use \noindent\textbf  rather than \paragraph
\setbool{compact}{true}

\usepackage{booktabs}       % professional-quality tables
\usepackage{amsfonts}       % blackboard math symbols
\usepackage{nicefrac}       % compact symbols for 1/2, etc.
\usepackage{xcolor}         % colors

\title{Logarithmic Regret in Feature-based Dynamic Pricing}

\author{%
	Jianyu Xu \\%\thanks{Use footnote for providing further information
	Department of Computer Science\\
	University of California, Santa Barbara\\
	Santa Barbara, CA 93106 \\
	\texttt{xu\_jy15@ucsb.edu} \\
	\And
	Yu-Xiang Wang \\
	Department of Computer Science \\
	University of California, Santa Barbara\\
	Santa Barbara, CA 93106 \\
	\texttt{yuxiangw@cs.ucsb.edu} \\
}

\begin{document}
	
\maketitle

\begin{abstract}
	Feature-based dynamic pricing is an increasingly popular model of setting prices for highly differentiated products with applications in digital marketing, online sales, real estate and so on. The problem was formally studied as an online learning problem \citep{javanmard2019dynamic} where a seller needs to propose prices \emph{on the fly} for a sequence of $T$ products based on their features $x$ while having a small \emph{regret} relative to the best ---``omniscient’’--- pricing strategy she could have come up with in hindsight. We revisit this problem and provide two algorithms (EMLP and ONSP) for stochastic and adversarial feature settings, respectively, and prove the optimal $O(d\log{T})$ regret bounds for both. In comparison, the best existing results are $O\left(\min\left\{\frac{1}{\lambda_{\min}^2}\log{T}, \sqrt{T}\right\}\right)$ and $O(T^{2/3})$ respectively, with $\lambda_{\min}$ being the smallest eigenvalue of $\mathbb{E}[xx^T]$ that could be arbitrarily close to $0$.  We also prove an $\Omega(\sqrt{T})$ information-theoretic lower bound for a slightly more general setting, which demonstrates that ``knowing-the-demand-curve'' leads to an exponential improvement in feature-based dynamic pricing. 
%Feature-based dynamic pricing is an increasingly popular model for setting prices for highly differentiated products with applications in digital marketing, online sales, real estates and so on. The problem was formally studied as an online learning problem \citep{ javanmard2019dynamic,cohen2020feature} where a seller needs to propose prices \emph{on the fly} for a sequence of $T$ products based on their features $x$ while having a small \emph{regret} relative to the best — ``omniscient’’ — pricing strategy she could have come up with in the hindsight. We revisit this problem and provide two algorithms (EMLP and ONSP) for stochastic and adversarial feature settings respectively, and prove each of them $O(d\log{T})$ regret. In comparison, the best existing results so far are $O\left(\min\left\{\frac{1}{\lambda_{\min}^2}\log{T}, \sqrt{T}\right\}\right)$ and $O(T^{2/3})$ respectively, with $\lambda_{\min}$ being the smallest eigenvalue of $\mathbb{E}[xx^T]$ that could be arbitrarily close to $0$. Besides, we also prove an information-theoretic lower bound under a slightly weaker assumption, to show in contrast the hardness of the current problem setting.
	\bigskip
	
	{\textbf{Keywords}: dynamic pricing, online learning, adversarial features, optimal regret, affine invariant, distribution-free.}
\end{abstract}
\newpage
\tableofcontents
\newpage

\section{Introduction}
\label{sec_introduction}
The problem of pricing — to find a high-and-acceptable price — has been studied since \citet{cournot1897researches}. In order to locate the optimal price that maximizes the revenue, a firm may adjust their prices of products frequently, which inspires the \emph{dynamic pricing} problem. Existing works \citep{kleinberg2003value,broder2012dynamic, chen2013simple, besbes2015surprising} primarily focus on pricing a single product, which usually will not work well in another setting when thousands of new products are being listed every day with no prior experience in selling them. Therefore, we seek methods that approach an acceptable-and-profitable price with only observations on this single product and some historical selling records of other products.

In this work, we consider a ``feature-based dynamic pricing'' problem, which was studied by \citet{amin2014repeated, cohen2020feature_journal, javanmard2019dynamic}. In this problem setting, a sales session (product, customer and other environmental variables) is described by a feature vector, and the customer's \emph{expected} valuation is modeled as a linear function of this feature vector.

%is characterized by its features, and the customers' expected valuation of a product is linearly dependent on the features.% Since customers are usually valuating a product by its features, it is rational to assume that the features are related to its demand function of price. More formally, a feature-based dynamic pricing agent interacts with buyers as follows.

\fbox{\parbox{0.97\textwidth}{Feature-based dynamic pricing. For $t=1,2,...,T:$
		\small
		\noindent
		\begin{enumerate}[leftmargin=*,align=left]
			\setlength{\itemsep}{0pt}
			\item A feature vector $x_t\in\R^{d}$ is revealed that describes a sales session (product, customer and context).			
			%			 with a feature vector $x_t\in\R^{d}$, $t=1,2,\ldots,T$, where $d$ is the dimension of features.
			\item The customer valuates the product as $w_t = x_t^{\top}\theta^{*} + N_t$.
			\item The seller proposes a price $v_t>0$ concurrently (according to $x_t$ and historical sales records).
			%A customer sees $x_t$ and evaluate it as $w_t = x_t^{\top}\theta^{*} + N_t$.
			\item The transaction is successful if $v_t  \leq w_t$, i.e., the seller gets a reward (payment) of $r_t=v_t\cdot\ind(v_t\leq w_t)$.%, where $\ind(A)=1$ if $A$ is true, and $0$ otherwise.%:
		\end{enumerate}
	}
}

Here $T$ is unknown to the seller (and thus can go to infinity), $x_t$'s can be either stochastic (e.g., each sales session is drawn i.i.d.) or adversarial (e.g., the sessions arrive in a strategic sequence), $\theta^{*}\in\R^{d}$ is a fixed parameter for all time periods, $N_t$ is a zero-mean noise, and $\ind_t=\ind(v_t\leq w_t)$ is an indicator that equals $1$ if $v_t\leq w_t$ and $0$ otherwise.  In this online-fashioned setting, we only see and sell one product at each time. Also, the feedback is \emph{Boolean Censored}, which means we can only observe $\ind_t$ instead of knowing $w_t$ directly. The best pricing policy for this problem is the one that maximizes the \emph{expected} reward, and the \emph{regret} of a pricing policy is accordingly defined as the difference of expected rewards between this selected policy and the best policy.

\ifbool{compact}{\noindent\textbf{Summary of Results.}}{\paragraph{Summary of Results.}} Our contributions are threefold.

\begin{enumerate}
	\item When $x_t$'s are independently and identically distributed (i.i.d.) from an unknown distribution, we propose an ``Epoch-based Max-Likelihood Pricing (EMLP)'' algorithm that guarantees a regret bound at $O(d\log{T})$. The design of EMLP is similar to that of the RMLP algorithm in \citet{javanmard2019dynamic}, but our new analysis improves their regret bound at $O(\sqrt{T})$ when $\E[xx^{\top}]$ is near singular. %$O(\frac{1}{\lambda_{\min}^2}\log{T}, \sqrt{T})$, where $\lambda_{\min}$ is the smallest eigen value of $\E[xx^{\top}]$ that could be arbitrarily to $0$. % by getting rid of constraints on feature distributions.
	\item When $x_t$'s are adversarial, we propose an ``Online-Newton-Step Pricing (ONSP)'' algorithm that achieves $O(d\log{T})$ regret on constant-level noises for the first time, which exponentially improves the best existing result of $O(T^{2/3})$ \citep{cohen2020feature_journal}.\footnote{Previous works \citep{cohen2020feature_journal, krishnamurthy2020contextual} did achieve polylog regrets, but only for negligible noise with $\sigma=O(\frac{1}{T\log{T}})$.}  % In comparison, existing works \citep{cohen2020feature, krishnamurthy2020contextual} achieve similar regrets only on negligible small regrets with $\sigma=O(\frac{1}{T\log{T}})$. % 
	\item %Inspired by \citet{broder2012dynamic}, we looked into our technical assumption of Gaussian noise distribution, and 
	Our methods that achieve logarithmic regret require knowing the exact distribution of $N_t$ in advance, as is also assumed in \citet{javanmard2019dynamic}.
	We prove an $\Omega(\sqrt{T})$ lower bound on the regret if $N_t\sim\cN(0,\sigma^2)$ where $\sigma$ is \emph{unknown}, even with $\theta^*$ given and $x_t$ fixed for all $t$. 
	%we are not provided with the standard deviation $\sigma$ in advance. In fact, a bias of $O(T^{-\frac14})$ on $\sigma$ can lead to such a worst case. 
\end{enumerate}

%From the perspective of characterizing the hardness of dynamic pricing problems,
The $O(\log{T})$ regret of EMLP and ONSP meets the information-theoretical lower bound \citep[Theorem 5,][]{javanmard2019dynamic}. In fact, the bound is optimal even when $w_t$ is revealed to the learner \citep{mourtada2019exact}. From the perspective of characterizing the hardness of dynamic pricing problems, we generalize the classical results on ``The Value of Knowing a Demand Curve'' \citep{kleinberg2003value} by further dividing the random-valuation class with an exponential separation of: (1) $O(\log{T})$ regret for knowing the \emph{demand curve} exactly (even with adversarial features), and (2) $\Omega(\sqrt{T})$ regret for 
\emph{almost} knowing the \emph{demand curves} (up to a one-parameter parametric family). 

% Our main results indicate that any adversary cannot break the regret guarantee as the model is totally parameterized, but a tiny mismatch of noise parameter $\sigma$ can drastically escalate the hardness of this online pricing problem.  

%In this work, we firstly consider the case when each feature $x_t$ comes independently and identically distributed (i.i.d.), and propose a ``Epoch-based Max-Likelihood Pricing (EMLP)'' algorithm that guarantees a regret bound of $O(d\log{T})$. This algorithmic design is similar to the RMLP algorithm in \citet{javanmard2019dynamic}, but we analyze it in a new sight and improve their regret bound by getting rid of constraints on feature distributions. Meanwhile, for the general case when $x_t$'s are adversarially chosen, we propose an ``Online-Newton-Step Pricing (ONSP)'' algorithm  that also guarantees $O(d\log{T})$ regret bound. On the contrary, the best existing result for the adversarial setting is $O(T^{2/3})$ by \citet{cohen2020feature}.

%After that, we look into our technical assumption of Gaussian noise distribution, and prove an $\Omega(\sqrt{T})$ lower regret bound if we are not provided with the standard deviation $\sigma$ in advance. Last but not least, we conduct numerical experiments to validate our algorithm, along with an additional comparison with existing contextual bandit methods.

%\section{Related Works}

\section{Related Works}
\vspace{-0.5em}
\label{sec_related_works}
In this section, we discuss our results relative to existing works on feature-based dynamic pricing, and highlight the connections and differences to the related settings of contextual bandits and contextual search (for a broader discussion, see Appendix~\ref{appendix_other_related_works}).  

%mainly look into those closely related works on contextual bandits and feature-based dynamic pricing. Please refer to Appendix~\ref{appendix_other_related_works} for a broader discussion of related works. %For a broader discussion of the related works, please refer to Appendix~\ref{appendix_other_related_works} and the references therein.

\ifbool{compact}{\noindent\textbf{Feature-based Dynamic Pricing.}}{\paragraph{Feature-based Dynamic Pricing.}}
There is a growing body of work on dynamic pricing with linear features \citep{amin2014repeated,qiang2016dynamic,cohen2020feature_journal,javanmard2019dynamic}. Table \ref{table_related_works_and_regret_bounds} summarizes the differences in the settings and results\footnote{We only concern the dependence on $T$ since there are various different assumptions on $d$.}.
\begin{table*}[t]%[tbh]
	\centering
	\caption{Related Works and Regret Bounds w.r.t. $T$}
	\label{table_related_works_and_regret_bounds}
	\resizebox{\textwidth}{!}{
		\begin{tabular}{|c|c|c|c|p{2.5in}|c|}
			\hline
			Algorithm & Work                     & Regret (upper) bound         & Feature & Noise       \\ \hline
			LEAP      & \citep{amin2014repeated} & $\tilde{O}(T^{\frac{2}{3}})$ & i.i.d.  & Noise-free  \\ \hline
			EllipsoidPricing &  \citep{cohen2020feature_journal} & $O(\log{T})$         & adversarial & Noise-free \\ \hline
			EllipsoidEXP4 &  \citep{cohen2020feature_journal} & $\tilde{O}(T^{\frac{2}{3}})$ & adversarial & Sub-Gaussian  \\ \hline
			PricingSearch & \citep{leme2018contextual} & $O(\log\log(T))$ & adversarial & Noise-free \\ \hline
			\multirow{2}{*}{RMLP} & \multicolumn{1}{c|}{\multirow{2}{*}{ \citep{javanmard2019dynamic}}} & $O(\log{T}/C_{\min}^2)$$^\dagger$                 & \multirow{2}{*}{i.i.d.} & \multirow{2}{*}{\begin{tabular}[c]{@{}l@{}}Log-concave, distribution-known\end{tabular}} \\ \cline{3-3} \cline{6-6} 
			& \multicolumn{1}{c|}{}  & $O(\sqrt{T})$      &      &    \\ \hline
			RMLP-2  & \citep{javanmard2019dynamic}   & $O(\sqrt{T})$  & i.i.d.  & Known parametric family of log-concave. \\ \hline
			ShallowPricing & \citep{cohen2020feature_journal} & \multirow{2}{*}{$O(poly(\log{T}))$} & \multirow{2}{*}{adversarial} & \multirow{2}{*}{Sub-Gaussian, known $\sigma=O(\frac{1}{T\log{T}})$} \\ \cline{1-2}
			CorPV          & \citep{krishnamurthy2020contextual}   &                                            &                              &                                                            \\ \hline
			%CorPV & \citep{krishnamurthy2020contextual} & $\tilde{O}(d^3 poly\log T)$ & adversarial & Sub-Gaussian, $\sigma=O(\frac{1}{T\log T})$, distribution-known \\ \hline
			Algorithm 2 (MSPP) & \citep{liu2021optimal} & $O(\log\log(T))$ & adversarial & Noise-free \\ \hline
			\textbf{EMLP} & This paper & $O(\log T)$ & i.i.d. & Strictly log-concave, distribution-known \\ \hline
			\textbf{ONSP} & This paper & $O(\log{T})$ & adversarial & Strictly log-concave, distribution-known \\ \hline
	\end{tabular}}
	\footnotesize{$^\dagger$ $C_{\min}$ is the restricted eigenvalue condition. It reduces to the smallest eigenvalue of $\E[xx^{\top}]$ in the low-dimensional case we consider.}\\
\end{table*}
Among these work, our paper directly builds upon \citep{cohen2020feature_journal} and \citep{javanmard2019dynamic}, as we share the same setting of online feature vectors, linear and noisy valuations and Boolean-censored feedback.  Relative to the results in \citep{javanmard2019dynamic}, we obtain $O(d\log T)$ regret under weaker assumptions on the sequence of input features --- in both distribution-free stochastic feature setting and the adversarial feature setting. %In the cases when $\lambda_{\min}$ is close to $0$, the RMLP algorithm \citep{javanmard2019dynamic} is only proven to achieve $O(\sqrt{T})$. 
It is to be noted that \citep{javanmard2019dynamic} also covers the sparse high-dimensional setting, and handles a slightly broader class of demand curves. Relative to \citep{cohen2020feature_journal}, in which the adversarial feature-based dynamic pricing was first studied, 
our algorithm ONSP enjoys the optimal $O(d\log T)$ regret when the noise-level is a constant. In comparison, \citet{cohen2020feature_journal} reduces the problem to contextual bandits and applies the (computationally inefficient) ``EXP-4'' algorithm \citep{auer2002nonstochastic} to achieve a $\tilde{O}(T^{2/3})$ regret. The ``bisection'' style-algorithm in both \citet{cohen2020feature_journal} and \citet{krishnamurthy2020contextual} could achieve $\tilde{O}(poly(d) poly\log(T))$ regrets but requires a small-variance subgaussian noise satisfying $\sigma = O(\frac{1}{T\log{T}})$. 

%Earlier works \citep{amin2014repeated, qiang2016dynamic} also study linear feature-based dynamic pricing problems. \citet{amin2014repeated} describes an algorithm that achieves a weaker $\tilde{O}(T^{\frac{2}{3}})$ regret for the noiseless setting. \citet{qiang2016dynamic} assumes a linear and uncensored full information-feedback that is slightly different from our setting, and achieves an $O(\log{T})$ regret.

\ifbool{compact}{\noindent\textbf{Lower Bounds.}}{\paragraph{Lower Bounds.}} Most existing works focus on the lower regret bounds of non-feature-based models. \citet{kleinberg2003value} divides the problem setting as fixed, random, and adversarial valuations, and then proves each a $\Theta(\log\log{T})$, $\Theta(\sqrt{T})$, and $\Theta(T^{2/3})$ regret, respectively. \citet{broder2012dynamic} further proves a $\Theta(\sqrt{T})$ regret in general parametric valuation models. In this work, we generalize the methods of \citet{broder2012dynamic} to our feature-based setting and further narrow it down to a linear-feature Gaussian-noisy model. As a complement to \citet{kleinberg2003value}, we further separate the exponential regret gap between: (1) $O(\log{T})$ of the hardest (adversarial feature) totally-parametric model, and (2) $\Omega(\sqrt{T})$ of the simplest (fixed known expectation) unknown-$\sigma$ Gaussian model.

\ifbool{compact}{\noindent\textbf{Contextual Bandits.}}{\paragraph{Contextual Bandits.}}
For readers familiar with the online learning literature, our problem can be reduced to a contextual bandits problem \citep{langford2007epoch,agarwal2014taming} by discretizing the prices. But this reduction only results in $O(T^{2/3})$ regret, as it does not capture the special structure of the feedback: \emph{an accepted price indicates the acceptance of all lower prices}, and vise versa.  Moreover, when comparing to linear bandits \citep{chu2011contextual}, it is the valuation instead of the expected reward that we assume to be linear.

\ifbool{compact}{\noindent\textbf{Contextual Search.}}{\paragraph{Contextual Search.}}
Feature-based dynamic pricing is also related to the contextual search problem \citep{lobel2018multidimensional, leme2018contextual, liu2021optimal, krishnamurthy2020contextual}, which often involves learning from Boolean feedbacks, sometimes with a ``pricing loss'' and ``noisy'' feedback. These shared jargons make this problem \emph{appearing} very similar to our problem. However, except for the noiseless cases \citep{lobel2018multidimensional, leme2018contextual}, contextual search algorithms, even with ``pricing losses'' and ``Noisy Boolean feedback'' \citep[e.g.,][]{liu2021optimal}, do \emph{not} imply meaningful regret bounds in our problem setup due to several subtle but important differences in the problem settings. Specifically, the noisy-boolean feedback model of \citep{liu2021optimal} is about randomly toggling the ``purchase decision'' determined by the \emph{noiseless} valuation $x^{\top}\theta^*$ with probability $0.5-\epsilon$. This is incompatible to our problem setting where the purchasing decision is determined by a noisy valuation $x^{\top}\theta^* + \text{Noise}$.
 Ultimately, in the setting of \citep{liu2021optimal}, the optimal policy alway plays $x^{\top}\theta^*$, but our problem is harder in that we need to exploit the noise and the optimal price could be very different from $x^{\top}\theta^*$. \footnote{As an explicit example, suppose the valuation $x^{\top}\theta^*=0$, then the optimal price must be $>0$ in order to avoid zero return. }
\citet{krishnamurthy2020contextual} also discussed this issue explicitly and considered the more natural noisy Boolean feedback model studied in this paper. Their result, similar to \citet{cohen2020feature_journal}, only achieves a logarithmic regret when the noise on the valuation is vanishing in an $\tilde{O}(1/T)$ rate.

\section{Problem Setup}
\label{sec_preliminary}
\ifbool{compact}{\noindent\textbf{Symbols and Notations.}}{\paragraph{Symbols and Notations.}}
Now we introduce the mathematical symbols and notations involved in the following pages. The game consists of $T$ rounds. 
%To begin with, we denote $t=1,2,\ldots, T$ as round indices (also called ``periods'' \yw{``periods'' sound like you have multiple sessions. ``Round'' sounds better.}). 
%Accordingly, 
$x_{t}\in \R^d$, $v_{t}\in \R_+$ and  $N_{t}\in\R$ denote the feature vector, the proposed price and the noise respectively at round $t=1,2,...,T$.\footnote{In an epoch-design situation, a subscript $(k,t)$ indicates round $t$ of epoch $k$.} We denote the product $u_t:=x_t^{\top}\theta^{*}$ as an \emph{expected valuation}. At each round, we receive a payoff (reward) $r_{t}=v_t\cdot\ind_t$, where the binary variable $\ind_{t}$ indicates whether the price is accepted or not, i.e.,  $\ind_{t}=\mathbf{1}(v_t\leq w_t)$. As we may estimate $\theta^{*}$ in our algorithms, we denote $\hat{\theta}_t$ as an estimator of $\theta^{*}$, which we will formally define in the algorithms.
Furthermore, we denote some functions that are related to noise distribution: $F(\omega)$ and $f(\omega)$ denote the cumulative distribution function (CDF) and probability density function (PDF) sequentially. We know that $F'(\omega)=f(\omega)$ if we assume differentiability.
	%Furthermore, we denote some functions that are related to noise distribution: $\Phi_{\sigma}(\omega)$ denotes the cumulative distribution function (CDF) of a normal distribution $\cN(0, \sigma^2)$, and $p_{\sigma}(\omega)$, correspondingly, denotes the probability density function (PDF) of $\cN(0,\sigma^2)$. We also denote $p'_{\sigma}(\omega)=\frac{\text{d} p_{\sigma}(\omega)}{\text{d}\omega}$ for simplicity.  Specifically, we denote functions $\Phi_{1}, p_{1}, p'_{1}$ as those of standard ($\sigma=1$) normal distribution.
To concisely denote all data observed up to round $\tau$ (i.e., feature, price and payoff of all past rounds), we define $hist(\tau)=\{(x_{t}, v_{t}, \ind_{t}) \text{ for }t=1,2,...,\tau\}$. $hist(\tau)$ represents the \emph{transcript} of all observed random variables before round $(\tau+1)$.
	
	We define
	\begin{equation}
	l_{t}(\theta):=-\ind_t\cdot\log\big(1-F(v_t-x_t^{\top}\theta)\big)-(1-\ind_t)\log\big(F(v_t-x_t^{\top}\theta)\big)
	\label{equation_log_likelihood_function}
	\end{equation}
	%\begin{equation}
	%    l_{t}(\theta):=-\log\left(\Phi_{\sigma}\left((v_{t}-x_{t}^{\top}\theta)(-1)^{\ind_{t}}\right)\right)
	%    \label{equation_log_likelihood_function}
	%\end{equation}
	as a negative log-likelihood function at round $t$. Also, we define an expected log-likelihood function $L_t(\theta)$:
	\begin{equation}
	    L_t(\theta):=\E_{N_t}[l_t(\theta)|x_t]
	    \label{equation_expected_log_likelihood}
	\end{equation}
	Notice that we will later define an $\hat{L}_{k}(\theta)$ which is, however, not an expectation. % Usually, we denote it as $L(\theta)$ for a fixed $x$, without emphasizing $t$.

\ifbool{compact}{\noindent\textbf{Definitions of Key Quantities.}}{\paragraph{Definitions of Key Quantities.}}
We firstly define an \emph{expected reward} function $g(v,u)$.
\ifbool{twocol}{
\begin{equation}
\begin{aligned}
g(v, u):=&\mathbb{E}[r_t|v_t=v, x_t^{\top}\theta^{*}=u]\\
=&v\cdot P[v\leq x_t^{\top}\theta^{*}+N_t]\\
=&v\cdot(1-F(v-u)).
\end{aligned}
\label{equ_def_gvu}
\end{equation}
}{
\begin{equation}
    g(v, u):=\mathbb{E}[r_t|v_t=v, x_t^{\top}\theta^{*}=u]=v\cdot P[v\leq x_t^{\top}\theta^{*}+N_t]=v\cdot(1-F(v-u)).
    \label{equ_def_gvu}
\end{equation}
}
This indicates that if the expected valuation is $u$ and the proposed price is $v$, then the (conditionally) expected reward is $g(v,u)$. Now we formally define the \emph{regret} of a policy (algorithm) $\cA$ as is promised in Section \ref{sec_introduction}.

\begin{definition}[Regret]
	Let $\cA: \R^{d}\times\left(\R^{d},\R,\{0,1\}\right)^{t-1}\rightarrow\R$ be a policy of pricing, i.e. $\cA(x_t, hist(t-1)) = v_t$. The regret of $\cA$ is defined as follows.
	{\small
	\begin{equation}
	\begin{aligned}
	Reg_{\cA}%&=\sum_{t=1}^{T}\max\E[r_t|x_t]-\E[r_t(\cA(x_t,hist(t-1)))|x_t]\\
	=&\sum_{t=1}^{T}\max_{v}g(v,x_t^{\top}\theta^{*})-g(\cA(x_t,hist(t-1)), x_t^{\top}\theta^{*}).
	\end{aligned}
	\label{Reg_A}
	\end{equation}
}
	Here $hist(t-1)$ is the historical records until $(t-1)^{\text{th}}$ round.
\end{definition}

\ifbool{compact}{\noindent\textbf{Summary of Assumptions.}}{\paragraph{Summary of Assumptions.}}
We specify the problem settings by proposing three assumptions.

\begin{assumption}[Known, bounded, strictly log-concave distribution]
	The noise $N_t$ is independently and identically sampled from a distribution whose CDF is $F$. Assume that $F\in \mathbb{C}^{2}$ is strictly increasing and that $F$ and $(1-F)$ are strictly log-concave. Also assume that $f$ and $f'$ are bounded, and denote $B_f:=\sup_{\omega\in\R}{f(\omega)}, B_{f'}:=\sup_{\omega\in\R}|f'(\omega)|$ as two constants.
	%The noise $N_t\sim \cN(0,\sigma^2)$, where $\sigma>0$ is \emph{known} to us in advance. Also, $N_t$ are \emph{mutually independent}, $t=1,2,\ldots, T$.
	\label{assumption_strictly_log_concave}
\end{assumption}

\begin{assumption}[Bounded convex parameter space]\label{assumption_convex_feature_and_parameter_set}
	The true parameter $\theta^{*}\in\mathbb{H}$, where $\mathbb{H}\subseteq\{\theta: ||\theta||_2\leq B_1\}$ is a bounded convex set and $B_1$ is a constant. Assume $\mathbb{H}$ is known to us (but $\theta^*$ is not).
\end{assumption}

\begin{assumption}[Bounded feature space]
\label{assumption_bounded_x_set}
    Assume $x_t\in D\subseteq\{x: ||x||_2\leq B_2\}$, $\forall t=1,2,\ldots, T$. Also, $0\leq x^{\top}\theta\leq B, \forall x\in D, \forall\theta\in\mathbb{H}$, where $B=B_1\cdot B_2$ is a constant.
\end{assumption}

Assumption \ref{assumption_convex_feature_and_parameter_set} and \ref{assumption_bounded_x_set} are mild as we can choose $B_1$ and $B_2$ large enough. In Section \ref{subsec_alg_random}, we may add further complement to Assumption \ref{assumption_bounded_x_set} to form a stochastic setting. Assumption \ref{assumption_strictly_log_concave} is stronger since we might not know the exact CDF in practice, but it is still acceptable from an information-theoretic perspective. There are at least three reasons that lead to this assumption: Primarily, this is necessary if we hope to achieve an $O(\log(T))$ regret. We will prove in Section \ref{sec_lower} that an $\Omega(\sqrt{T})$ is unavoidable if we cannot know one parameter exactly. Moreover, the pioneering work of \citet{javanmard2019dynamic} also assumes a known noise distribution with log-concave CDF, and many common distributions are actually strictly log-concave, such as Gaussian and logistic.\footnote{In fact, $F$ and $(1-F)$ are both log-concave if its PDF is log-concave, according to Prekopa's Inequality.} Besides, although we did not present a method to precisely estimate $\sigma$ in this work, it is a reasonable algorithm to replace  with a plug-in estimator estimated using historical offline data. As we have shown, not knowing $\sigma$ requires $O(\sqrt{T})$ regret in general, but the lower bound does not rule out the plug-in approach achieving a smaller regret for interesting subclasses of problems in practice.

Finally, we state a lemma and define an argmax function helpful for our algorithm design. 

\begin{lemma}[Uniqueness]\label{g_unimodal}
		For any $u\geq0$, there exists a unique $v^{*}\geq0$ such that $g(v^{*},u)=\max_{v\in\R}{g}(v,u)$. Thus, we can define a \emph{greedily pricing} function that maximizes the expected reward:
		\begin{equation}
		J(u):=\mathop{\arg\max}_{v}g(v,u)
		\label{def_J}
		\end{equation}
\end{lemma}
\vspace{-1em}
Please see the proof of Lemma \ref{g_unimodal} in Appendix \ref{appendix_subsec_proof_lemma_g_unimodal}.% Since $g$ is unimodal, we can define a \emph{greedily pricing} function: % in Equation \eqref{def_J}. In other words, if the expected valuation is $u$, then a price $v=J(u)$ will maximize the expected reward.

\section{Algorithms}
\label{sec_algorithm}
\vspace{-0.5em}
In this section, we propose two dynamic pricing algorithms: EMLP and ONSP, for stochastic and adversarial features respectively.

\begin{figure}[t]
	\begin{minipage}{0.48\textwidth}
		\begin{algorithm}[H]
			\caption{Epoch-based max-likelihood pricing (EMLP)}
			\label{algo_MLP}
			%\SetAlgoLined
			
			\begin{algorithmic}
				\STATE {\bfseries Input:} {Convex and bounded set $\mathbb{H}$}
				\STATE Observe $x_1$, randomly choose $v_1$ and get $r_1$.
				\STATE Solve $\hat{\theta}_1=\arg\min_{\theta\in\mathbb{H}}l_1(\theta)$;
				\FOR{$k=1$ {\bfseries to} $\lfloor\log_2{T}\rfloor+1$}
				\STATE Set $\tau_k = 2^{k-1}$;
				\FOR{$t=1$ {\bfseries to} $\tau_k$ }
				\STATE Observe $x_{k,t}$;
				\STATE Set price $v_{k,t}=J(x_{k,t}^{\top}\hat{\theta}_{k}) $; %:=\mathop{\arg\max}\limits_{v}v\cdot \Phi_{\sigma}(x_{k,t}^{\top}\hat{\theta}_{k-1}-v)$;
				\STATE Receive $r_{k,t} = v_{k_t}\cdot\ind_t$;
				\ENDFOR
				\STATE Solve: $\hat{\theta}_{k+1}=\arg\min_{\theta\in\mathbb{H}}\hat{L}_{k}(\theta),$ where $\hat{L}_{k}(\theta)=\frac1{\tau_k}\sum_{t=1}^{\tau_k}l_{k,t}(\theta)$.
				\ENDFOR
			\end{algorithmic}
		\end{algorithm}
	\end{minipage}
	\hfill
	\begin{minipage}{0.48\textwidth}
		\begin{algorithm}[H]
			\caption{Online Newton Step Pricing (ONSP)}
			\label{algo_ONSP}
			\begin{algorithmic}
				\STATE {\bfseries Input:} Convex and bounded set $\mathbb{H}$, $\theta_1$, parameter $\gamma, \epsilon>0$
				\STATE Set $A_0=\epsilon\cdot{I_d}$;
				\FOR{$t=1$ {\bfseries to} $T$}
				\STATE Observe $x_t$;
				\STATE Set price $v_t=J(x_t^{\top}\theta_{t})$;
				\STATE Receive $r_t=v_t\cdot\ind_t$;
				\STATE Set surrogate loss function $l_t(\theta)$;
				\STATE Calculate $\nabla_t=\nabla{l}_t(\theta)$;
				\STATE Rank-1 update: $A_{t}=A_{t-1}+\nabla_t\nabla_t^{\top}$;
				\STATE Newton step: $\hat{\theta}_{t+1} = \theta_t-\frac{1}{\gamma}A_{t}^{-1}\nabla_t$;
				\STATE Projection: $\theta_{t+1}=\prod_{\mathbb{H}}^{A_t}(\hat{\theta}_{t+1})$.
				\ENDFOR
				%\REPEAT
				%\IF{$x_i > x_{i+1}$}
				%\ENDIF
				%\UNTIL{$noChange$ is $true$}
			\end{algorithmic}
		\end{algorithm}
	\end{minipage}
	\vspace{-1em}
\end{figure}

\subsection{Pricing with Distribution-Free Stochastic Features}\label{subsec_alg_random}
\vspace{-0.5em}
\begin{assumption}[Stochastic features]\label{random_iid_xt}
	Assume $x_t\sim \mathbb{D}\subseteq D$ are independently identically distributed (i.i.d.) from an unknown distribution, for any $t=1,2,\ldots,T$.%, where $\mathbb{D}$ is a convex set.
\end{assumption}

%Let us now define some notations specifically for EMLP and its analysis. We denote the index of epoch as $k$, the length of epoch $k$ as $\tau_k$, and $hist(k)$ as the transcripts for epoch $1\sim k$.
%We then define an \emph{empirical negative-log likelihood function} 
%$\hat{L}_{k}(\theta):=\frac{1}{\tau_k}\sum_{t=1}^{\tau_k}l_{k,t}(\theta).$
% and its expectation $L_{k}(\theta):=\E_{h}[l_{k,t}(\theta)]$, where $\E_{h}[\cdot]:=\E[\cdot|hist(k-1)]$.

The first algorithm, Epoch-based Max-Likelihood Pricing (EMLP) algorithm, is suitable for a stochastic setting defined by Assumption \ref{random_iid_xt}. EMLP proceeds in epochs with each stage doubling the length of the previous epoch. At the end of each epoch, we consolidate the observed data and solve a maximum likelihood estimation problem to learn $\theta$. A max likelihood estimator (MLE) obtained by minimizing
$\hat{L}_{k}(\theta):=\frac{1}{\tau_k}\sum_{t=1}^{\tau_k}l_{k,t}(\theta),$
which is then used in the next epoch as if it is the true parameter vector.  In the equation,  $k,\tau_k$ denotes the index and length of epoch $k$. The estimator is computed using data in $hist(k)$, which denotes the transcript for epoch $1\sim k$.
The pseudo-code of EMLP is summarized in Algorithm~\ref{algo_MLP}. In the remainder of this section, we discuss the computational efficiency and prove the upper regret bound of $O(d\log{T})$.

%EMLP enjoys a number of notable properties $x_t$. %that facilitate our calculation as well as regret analysis.

\ifbool{compact}{\noindent\textbf{Computational Efficiency.}}{\paragraph{Computational Efficiency.}} {The calculations in EMLP are straightforward except for $\arg\min\hat{L}_k(\theta)$ and $J(u)$.
%  In general, an arg-max(min) function could be computationally consuming due to the ``Curse of Dimensionality''. However,
As $g(v,u)$ is proved unimodal in Lemma \ref{g_unimodal}, we may efficiently calculate $J(u)$ by binary search. We will prove that $l_{k,t}$ is exp-concave (and thus also convex). Therefore, we may apply any off-the-shelf tools for solving convex optimization.}

\ifbool{compact}{\noindent\textbf{MLE and Probit Regression.}}{\paragraph{MLE and Probit Regression.}}
{ A closer inspection reveals that this log-likelihood function corresponds to a probit \citep{aldrich1984linear} or a logit model \citep{wright1995logistic} for Gaussian or logistic noises. See Appendix \ref{appendix_probit_regression}. % for more detials. %In fact, a closer inspection reveals that this log-likelihood function corresponds to precisely a probit regression model \citep{aldrich1984linear}, which is described as follows: a Boolean random variable $Y$ satisfies the following probabilistic distribution: $\P[Y=1|X]=F(X^{\top}\beta)$, where $X\in\mathbb{R}$ is a random vector and $\beta\in\mathbb{R}$ is a parameter. In our problem, we may treat $\mathbbm{1}_{k, t}$ as $Y$, $[{x_{k,t}}^{\top}, v_{k,t}]^{\top}$ as $X$ and $[{\theta^{*}}^{\top}, -1]^{\top}$ as $\beta$, which exactly fits the probit model. Therefore,  $\hat{\theta}_k=\mathop{\arg\min}_{\theta}\hat{L}_{k}(\theta)$ can be solved via the highly efficient implementation of generalized linear models, e.g., GLMnet, rather than resorting to generic tools for convex programming. As a heuristic, we could leverage the vast body of statistical work on probit models and adopt a fully Bayesian approach that jointly estimates $\theta$ and hyper-parameters of $F$. This would make the algorithm more practical by eliminating the need to choose the hyper-parameters when running this algorithm.

}

\ifbool{compact}{\noindent\textbf{Affine Invariance.}}{\paragraph{Affine Invariance.}} {Both optimization problems involved depend only on $x^{\top} \theta$, so if we add any affine transformation to $x$ into $\tilde{x} = Ax$, the agent can instead learn a new parameter of $\tilde{\theta}^{*}=(A^{\top})^{-1}\theta^{*}$ and achieve the same $u_t=x_t^{\top}\theta^{*}$. Also, the regret bound is not affected as the upper bound $B$ over $x^{\top}\theta$ does not change \footnote{Here $A$ is assumed invertible, otherwise the mapping from $\tilde{x}_t$ to $u_t$ does not necessarily exist.}. Therefore, it is only natural that the regret bound does not depend on the distribution $x$, nor the condition numbers of $\E[xx^{\top}]$ (i.e., the ratio of $\lambda_{\max}/\lambda_{\min}$).}

%\noindent\textbf{Arg-max Oracle $J(u)$.}
%{Though we cannot write $J(u)$ as an explicit function, we can still calculate it within polynomial time (with certain precision). This is because $g(v,u)$ is unimodal with respect to $v$, for any given $u$, and we may apply a ``method of bisection'' to numerically finds the optimal solution in logarithmic time.}
\subsection{Pricing with Adversarial Features}\label{subsec_alg_adversarial}

In this part, we propose an ``Online Newton Step Pricing (ONSP)'' algorithm that deals with adversarial $\{x_t\}$ series and guarantees $O(d\log{T})$ regret. The pseudo-code of ONSP is shown as Algorithm \ref{algo_ONSP}. In each round, it uses the likelihood function as a surrogate loss and applies ``Online Newton Step''(ONS) method to update $\hat{\theta}$. In the next round, it adopts the updated $\hat{\theta}$ and sets a price greedily. In the remainder of this section, we discuss some properties of ONSP and prove the regret bound.

The calculations of ONSP are straightforward. The time complexity of calculating the matrix inverse $A_{t}^{-1}$ is $O(d^3)$, which is fair as $d$ is small. In high-dimensional cases, we may use \emph{Woodbury matrix identity}\footnote{$(A+xx^{\top})^{-1}=A^{-1}-\frac{1}{1+x^{\top}A^{-1}x}A^{-1}x(A^{-1}x)^{\top} .$} to reduce it to $O(d^2)$ as we could get $A^{-1}$ directly from the latest round.

\section{Regret Analysis}
\label{sec_regret}
In this section, we mainly prove the logarithmic regret bounds of EMLP and ONSP corresponding to stochastic and adversarial settings, respectively. Besides, we also prove an $\Omega(\sqrt{T})$ regret bound on fully parametric $F$ with one parameter unknown.

\subsection{\texorpdfstring{$O(d\log{T})$}{Lg} Regret of EMLP}\label{subsec_random_analysis}

In this part, we present the regret analysis of Algorithm \ref{algo_MLP}. First of all, we propose the following theorem as our main result on EMLP.

\begin{theorem}[Overall regret]
	With Assumptions \ref{assumption_strictly_log_concave}, \ref{assumption_convex_feature_and_parameter_set}, \ref{assumption_bounded_x_set} and \ref{random_iid_xt}, the expected regret of EMLP can be bounded by:
	\begin{equation}
	\mathbb{E}[Reg_{\text{EMLP}}]\leq 2C_{s}{d}{\log{T}},
	\end{equation}
	where $C_{s}$ is a constant that depends only on $F(\omega)$ and is independent to $\mathbb{D}$.
	\label{corollary_expected regret bound}
\end{theorem}

The proof of Theorem \ref{corollary_expected regret bound} is sophisticated. For the sake of clarity, we next present an inequality system as a roadmap toward the proof. After this, we formally illustrate each line of it with lemmas.

Since EMLP proposes $J(x_{k,t}^{\top}\hat{\theta}_{k})$ in every round of epoch $k$, we may denote the per-round regret as $Reg_t(\hat{\theta}_{k})$, where:
\begin{equation}
Reg_t(\theta): =g(J(x_t^{\top}\theta^{*}), x_t^{\top}\theta^{*})-g(J(x_t^{\top}\theta), x_t^{\top}\theta^{*}).
\label{eqn_parameter_based_regret}
\end{equation}
Therefore, it is sufficient to prove the following Theorem:

\begin{theorem}[Expected per-round regret] \label{expected regret bound}
	For the per-round regret defined in Equation \eqref{eqn_parameter_based_regret}, we have:
	\begin{equation*}
	\mathbb{E}[Reg_{k,t}(\hat{\theta}_k)]\leq C_{s}\cdot\frac{d}{\tau_k}.
	\end{equation*}
\end{theorem}
The proof roadmap of Theorem \ref{expected regret bound} can be written as the following inequality system.
\ifbool{twocol}{
\begin{equation}
\begin{aligned}
\E[Reg_{k,t}(\hat{\theta}_k)]&\leq C\cdot\E[ (\hat{\theta}_k-\theta^{*})^{\top}x_{k,t}x_{k,t}^{\top}(\hat{\theta}_k-\theta^{*})]\\
&\leq\frac{2C}{C_{\text{down}}}\E[L_k(\hat{\theta}_k)-L_k(\theta^{*})]\\
&\leq\frac{2C\cdot C_{\text{exp}}}{C_{\text{down}}^2}\frac{d}{\tau_k}.
\end{aligned}
\label{inequ_sys_reg_mle}
\end{equation}
}{
\begin{equation}
\E[Reg_{k,t}(\hat{\theta}_k)]\leq C\cdot\E[ (\hat{\theta}_k-\theta^{*})^{\top}x_{k,t}x_{k,t}^{\top}(\hat{\theta}_k-\theta^{*})]\leq\frac{2C}{C_{\text{down}}}\E[\hat{L}_k(\hat{\theta}_k)-\hat{L}_k(\theta^{*})]\leq\frac{2C\cdot C_{\text{exp}}}{C_{\text{down}}^2}\frac{d}{\tau_k}.
\label{inequ_sys_reg_mle}
\end{equation}
}

We explain Equation \eqref{inequ_sys_reg_mle} in details. The first inequality comes from the following Lemma \ref{analytical_regret_bound}.

\begin{lemma}[Quadratic regret bound]
	%For $Reg(\theta)$, we have:
	We have:
	\ifbool{twocol}{
		\begin{equation}
		Reg_t(\theta)\leq C\cdot (\theta-\theta^{*})^{\top}x_tx_t^{\top}(\theta-\theta^{*}),
		\end{equation}$\forall\theta\in\mathbb{H}, \forall x_t\in\mathbb{D}$.
	}{
	\begin{equation}
	Reg_t(\theta)\leq C\cdot (\theta-\theta^{*})^{\top}x_tx_t^{\top}(\theta-\theta^{*}), \forall\theta\in\mathbb{H}, \forall x_t\in\mathbb{D}.
	\end{equation}
	} Here $C=2B_f+(B+J(0))\cdot B_{f'}.$
	\label{analytical_regret_bound}
\end{lemma}

The intuition is that function $g(J(u),u)$ is $2^{nd}$-order-smooth at $(J(u^*), u^*)$. A detailed proof of Lemma \ref{analytical_regret_bound} is in Appendix \ref{appendix_subsec_proof_of_regret_bound}. Note that $C$ is highly dependent on the distribution $F$. After this, we propose Lemma \ref{lemma_quadratic_likelihood_bound} that contributes to the second inequality of Equation \eqref{inequ_sys_reg_mle}.

\begin{lemma}[Quadratic likelihood bound]
	For the expected likelihood function $L_t(\theta)$ defined in Equation \eqref{equation_expected_log_likelihood}, we have:
		\begin{equation}
		L_t(\theta)-L_t(\theta^{*})\geq\frac12 C_{\text{down}}(\theta-\theta^{*})^{\top}x_tx_t^{\top}(\theta-\theta^{*}), \forall \theta\in\mathbb{H}, \forall x\in\mathbb{D},
		\label{equation_quadratic_likelihood_bound}
		\end{equation}
		\begin{equation}
		\text{ 		where  }\;\;\;
		C_{\text{down}}:=\inf_{\omega\in[-B, B+J(0)]}\min\left\{\frac{\text{d}^2\log(1-F(\omega))}{\text{d}\omega^2}, \frac{\text{d}^2\log(F(\omega))}{\text{d}\omega^2}\right\}>0.
		\label{eq__C_down}
		\end{equation}
	\label{lemma_quadratic_likelihood_bound}
\end{lemma}

\begin{proof}
	Since the true parameter always maximizes the expected likelihood function \citep{murphy2012machine}, by Taylor Expansion we have $\nabla L(\theta^{*})=0$, and hence $L_t(\theta)-L_t(\theta^{*}) = \frac{1}{2}(\theta-\theta^{*})^{\top}\nabla^2 L_t(\tilde{\theta})(\theta-\theta^{*})$ for some $\tilde{\theta}=\alpha\theta^{*}+(1-\alpha)\theta$. Therefore, we only need to prove the following lemma:
	
	\begin{lemma}[Strong convexity and Exponential Concavity]\label{strong_convexity}%%Theorem 4: Strongly convexity
		Suppose $l_{t}(\theta)$ is the negative log-likelihood function in epoch $k$ at time $t$. For any $\theta\in\mathbb{H}, x_{t}\sim\mathbb{D}$, we have:
		\ifbool{twocol}{
			\begin{equation}
			\begin{aligned}
			\nabla^2l_{t}(\theta)&\succeq C_{\text{down}} x_{t}x_{t}^{\top}\\
			&\succeq \frac{C_{\text{down}}} {C_{\text{exp}}}\nabla l_{t}(\theta)\nabla l_{t}(\theta)^{\top}\succeq 0,
			\end{aligned}
			\end{equation}
		}{
			\begin{equation}
			\nabla^2l_{t}(\theta)\succeq C_{\text{down}} x_{t}x_{t}^{\top}\succeq \frac{C_{\text{down}}} {C_{\text{exp}}}\nabla l_{t}(\theta)\nabla l_{t}(\theta)^{\top}\succeq 0,
			\end{equation}
		}
		\begin{equation}
		\text{	where }\;
		C_{\text{exp}}:=\sup_{\omega\in[-B, B+J(0)]}\max\left\{\frac{f(\omega)^2}{F(\omega)^2}, \frac{f(\omega)^2}{(1-F(\omega))^2}\right\}<+\infty.
		\label{eqn_def_c_exp}
		\end{equation}
	\end{lemma}
	Proof of Lemma \ref{strong_convexity} is in Appendix \ref{appendix_subsec_proof_of_strong_convexity}. With this lemma,  we see that Lemma \ref{lemma_quadratic_likelihood_bound} holds.
\end{proof}
With Lemma \ref{analytical_regret_bound} and Lemma \ref{lemma_quadratic_likelihood_bound}, we can immediately get the following Lemma \ref{theorem_surrogate}.

\begin{lemma}[Surrogate Regret]
	The relationship between $Reg(\theta)$ and likelihood function can be shown as follows:
	\begin{equation}
	Reg_t(\theta)\leq \frac{2\cdot C}{C_{\text{down}}}\left(L_t(\theta)-L_t(\theta^{*})\right),
	\label{equation_surrogate}
	\end{equation}
	$\forall\theta\in\mathbb{H},\forall{x}\in\mathbb{D}$, where $C$ and
	$C_{\text{down}}$ are defined in Lemma \ref{analytical_regret_bound} and \ref{lemma_quadratic_likelihood_bound} respectively.
	\label{theorem_surrogate}
\end{lemma}

Lemma \ref{theorem_surrogate} enables us to choose the negative log-likelihood function as a surrogate loss. This is not only an important insight of EMLP regret analysis, but also the foundation of ONSP design.

The last inequality of Equation \eqref{inequ_sys_reg_mle} comes from this lemma:

\begin{lemma}[Per-epoch surrogate regret bound]%%%\Delta L < = O(\frac{d}{\tau})
	Denoting $\hat{\theta}_k$ as the estimator coming from epoch $(k-1)$ and being used in epoch $k$, we have:
	\begin{equation}
	\mathbb{E}_{h}[\hat{L}_k(\hat{\theta}_k)-\hat{L}_k(\theta^{*})]\leq\frac{C_{\text{exp}}}{C_{\text{down}}}\cdot\frac{d}{\tau_k+1}.
	\label{d_over_tau_geq_diff_L}
	\end{equation}
	Here $C_{\text{exp}}$ is defined in Equation \ref{eqn_def_c_exp}, and $\E_h[\cdot] = \E[\cdot|hist(k-1)].$
	\label{before_expectation}
\end{lemma}
Proof of Lemma \ref{before_expectation} is partly derived from the work \citet{koren2015fast}, and here we give a proof sketch without specific derivations. A detailed proof lies in Appendix \ref{appendix_subsec_proof_of_uniform_stability}.

\noindent\emph{Proof sketch} of Lemma \ref{before_expectation}. We list the four main points that contribute to the proof:
\begin{itemize}
	\topsep-1em
	\itemsep0em
	\item Notice that $l_{k,t}(\theta)$ is strongly convex w.r.t. a seminorm $x_{k,t}x_{k,t}^{\top}$, we know $\hat{L}_{k}(\theta)$ is also strongly convex w.r.t. $\sum_{t=1}^{\tau_k}x_{k,t}x_{k,t}^{\top}$.
	\item For two strongly convex functions $g_1$ and $g_2$, we can upper bound the distance between their arg-minimals (scaled by some norm $||\cdot||$) with the dual norm of $\nabla(g_1-g_2)$.
	\item Since a seminorm has no dual norm, we apply two methods to convert it into a norm: (1) separation of parameters and likelihood functions with a ``leave-one-out'' method (to separately take expectations), and (2) separation of the spinning space and the null space.
	\item As the dual data-dependent norm offsets the sum of $xx^{\top}$ to a constant, Lemma \ref{before_expectation} holds.
\end{itemize}

We have so far proved Inequality \eqref{inequ_sys_reg_mle} after proving Lemma \ref{analytical_regret_bound}, \ref{lemma_quadratic_likelihood_bound}, \ref{before_expectation}. Therefore, Theorem \ref{expected regret bound} holds. %  

\subsection{\texorpdfstring{$O(d\log{T})$}{Lg} Regret of ONSP}\label{subsec_adversarial_analysis}
Here we present the regret analysis of Algorithm \ref{algo_ONSP} (ONSP). Firstly, we state the main theorem.

\begin{theorem}\label{theorem_ONSP_main}
	With Assumptions \ref{assumption_strictly_log_concave}, \ref{assumption_convex_feature_and_parameter_set}, \ref{assumption_bounded_x_set}, the regret of Algorithm \ref{algo_ONSP} (ONSP) satisfies:
	\begin{equation}
	Reg_{\text{ONSP}}\leq C_{a}\cdot d\log{T},
	\label{equation_ONSP_regret}
	\end{equation}
	where $C_{a}$ is a function only dependent on $F$.
\end{theorem}

\begin{proof}
	Proof of Theorem \ref{theorem_ONSP_main} here is more concise than Section \ref{subsec_random_analysis}, because the important Lemma \ref{strong_convexity} and \ref{theorem_surrogate} have been proved there. From Lemma \ref{theorem_surrogate}, we have:
	\begin{equation}
	g(J(u_t^*),u_t^{*})-g(J(u_t),u_t^{*})\leq \frac{2\cdot C}{C_{\text{down}}}\cdot\E_{N_t}[l_t(\theta_{t})-l_t(\theta^{*})].
	\label{equation_theorem_12}
	\end{equation}
	
	With Equation \ref{equation_theorem_12}, we may reduce the regret of likelihood functions as a surrogate regret of pricing. From Lemma \ref{strong_convexity} we see that the log-likelihood function is $\frac{C_{\text{down}}}{C_{\text{exp}}}$-exponentially concave\footnote{A function $f(\mu)$ is $\alpha$-exponentially concave iff $\nabla^2f(\mu)\succeq\alpha\nabla f(\mu)\nabla f(\mu)^{\top}$.}. This enables an application of Online Newton Step method to achieve a logarithmic regret. Therefore,  by citing from the \emph{Online Convex Optimization} \citep{hazan2019introduction}, we have the following Lemma.
	\begin{lemma}[Online Newton Step] \label{thmons} With parameters $\gamma = \frac{1}{2}\min\{\frac{1}{4GD}, \alpha\}$ and $\epsilon=\frac{1}{\gamma^2D^2}$, and $T>4$ guarantees:
		\begin{equation*}
		\sup_{\{x_t\}}\left\{\sum_{t=1}^{T}l_t(\theta_{t})-\min_{\theta\in\mathbb{H}}\sum_{t=1}^{T}l_t(\theta)\right\}\leq5\left(\frac{1}{\alpha}+GD\right)d\log T.
		\end{equation*}
		Here $\alpha = \frac{C_{\text{down}}}{C_{\text{exp}}}$, $D=2\cdot B_1$ and $G=\sqrt{C_{\text{exp}}}\cdot B_2$.
	\end{lemma}
		With Equation \ref{equation_theorem_12} and Lemma \ref{thmons}, we have:
		\begin{equation}
		Reg=\sum_{t=1}^{T}\left(g(J(u_t^*),u_t^{*})-\E_{N_1, N_2, \ldots, N_{t-1}}[g(J(u_t),u_t^{*})]\right)\leq\frac{2\cdot C}{C_{\text{down}}}\cdot5\left(\frac{1}{\alpha}+GD\right)d\log T.
		\end{equation}
	%\begin{equation}
	%\begin{aligned}
	%Reg=&\sum_{t=1}^{T}\left(g(J(u_t^*),u_t^{*})-\E_{N_1, N_2, \ldots, N_{t-1}}[g(J(u_t),u_t^{*})]\right)\\
	%\leq&\sum_{t=1}^{T}\left(\frac{2\cdot C}{C_{\text{down}}}\cdot\E_{N_1, N_2, \ldots, N_{t-1}}[\E_{N_t}[l_t(\theta_{t})-l_t(\theta^{*})]]\right)\\
	%\leq&\frac{2\cdot C}{C_{\text{down}}}\cdot\E_{N_1,\ldots, N_{t-1}}[\E_{N_t}[\sum_{t=1}^{T}l_t(\theta_t)-\min_{\theta\in\Omega}\sum_{t=1}^{T}l_t(\theta)]]\\
	%\leq&\frac{2\cdot C}{C_{\text{down}}}\cdot5\left(\frac{1}{\alpha}+GD\right)d\log T.
	%\end{aligned}
	%\end{equation}
	Therefore, we have proved Lemma \ref{theorem_ONSP_main}.
\end{proof}

\subsection{Lower Bound for Unknown Distribution}\label{sec_lower}

In this part, we evaluate Assumption \ref{assumption_strictly_log_concave} and prove that an $\Omega{(\sqrt{T})}$ lower regret bound is unavoidable with even a slight relaxation: a Gaussian noise with unknown $\sigma$. Our proof is inspired by \citet{broder2012dynamic} Theorem 3.1, while our lower bound relies on more specific assumptions (and thus applies to more general cases).

We firstly state Assumption \ref{assumption_fixed_unknown_parameter} covering this part, and then state Theorem \ref{theorem_lower_bound_square_root_t} as a lower bound:

\begin{assumption}
	The noise $N_t\sim\cN(0,\sigma^2)$ independently, where $0<\sigma\leq 1$ is fixed and \textbf{unknown}.
	\label{assumption_fixed_unknown_parameter}
\end{assumption}

\begin{theorem}[Lower bound with unknown $\sigma$]
	\label{theorem_lower_bound_square_root_t}
Under Assumption \ref{assumption_convex_feature_and_parameter_set}, \ref{assumption_bounded_x_set}, \ref{random_iid_xt} and \ref{assumption_fixed_unknown_parameter}, for any policy (algorithm) $\Psi: \R^{d}\times\left(\R^{d},\R,\{0,1\}\right)^{t-1}\rightarrow\R^{+}$ and any $T>2$, there exists a Gaussian parameter $\sigma\in\R^{+}$, a distribution $\mathbb{D}$ of features and a fixed parameter $\theta^{*}$, such that:
%	\begin{equation}
	$Reg_{\Psi}\geq{\frac1{24000}}\cdot\sqrt{T}.$
%	\end{equation}
\end{theorem}
\noindent\emph{Remark:} Here we assume $x_t$ to be i.i.d., which also implies the applicability on adversarial features. However, the minimax regret of the stochastic feature setting is $\Theta(\sqrt{T})$ \citep{javanmard2019dynamic}, while existing results have not yet closed the gap in adversarial feature settings.
% for the stochastic feature setting, the upper regret bound is also $O(\sqrt{T})$ (guaranteed by the RMLP-2 algorithm in \citet{javanmard2019dynamic}), while existing results have not yet close the gap for the adversarial feature setting.

\noindent\emph{Proof sketch} of Theorem \ref{theorem_lower_bound_square_root_t}. Here we assume a fixed valuation, i.e. $u^{*}=x_t^{\top}\theta^{*}, \forall t=1,2,\ldots$. Equivalently, we assume a fixed feature. The main idea of proof is similar to that in \citet{broder2012dynamic}: we assume $\sigma_1=1, \sigma_2 = 1-T^{-\frac14}$, and we prove that: (1) it is costly for an algorithm to perform well in both cases if the $\sigma$'s are different by a lot, and (2) it is costly for an algorithm to distinguish the two cases if $\sigma$'s are close enough to each other. We put the detailed proof in Appendix \ref{appendix_proof_lower_bound}.

%Since this part is not directly related to our main topic, we put the detailed proof in Appendix \ref{appendix_proof_lower_bound}.

\section{Numerical Result}
\label{sec_num_result}

\begin{figure*}[t]
	\centering
	\begin{subfigure}[t]{0.48\textwidth}
		\centering
		\includegraphics[width= \textwidth]{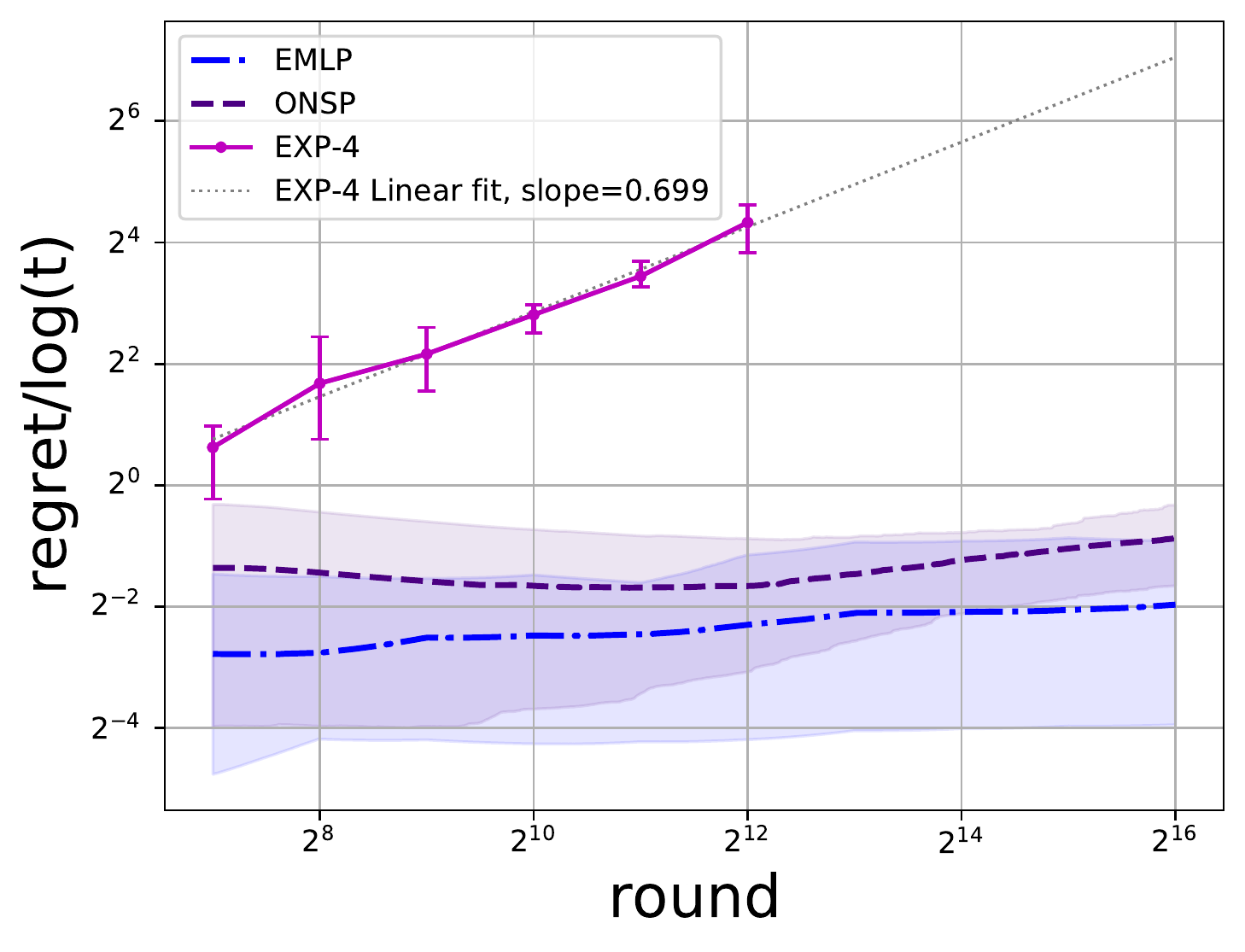}
		\caption{Stochastic feature}\label{stochastic_plot}
	\end{subfigure}
\quad	\begin{subfigure}[t]{0.48\textwidth}
		\centering
		\includegraphics[width=\textwidth]{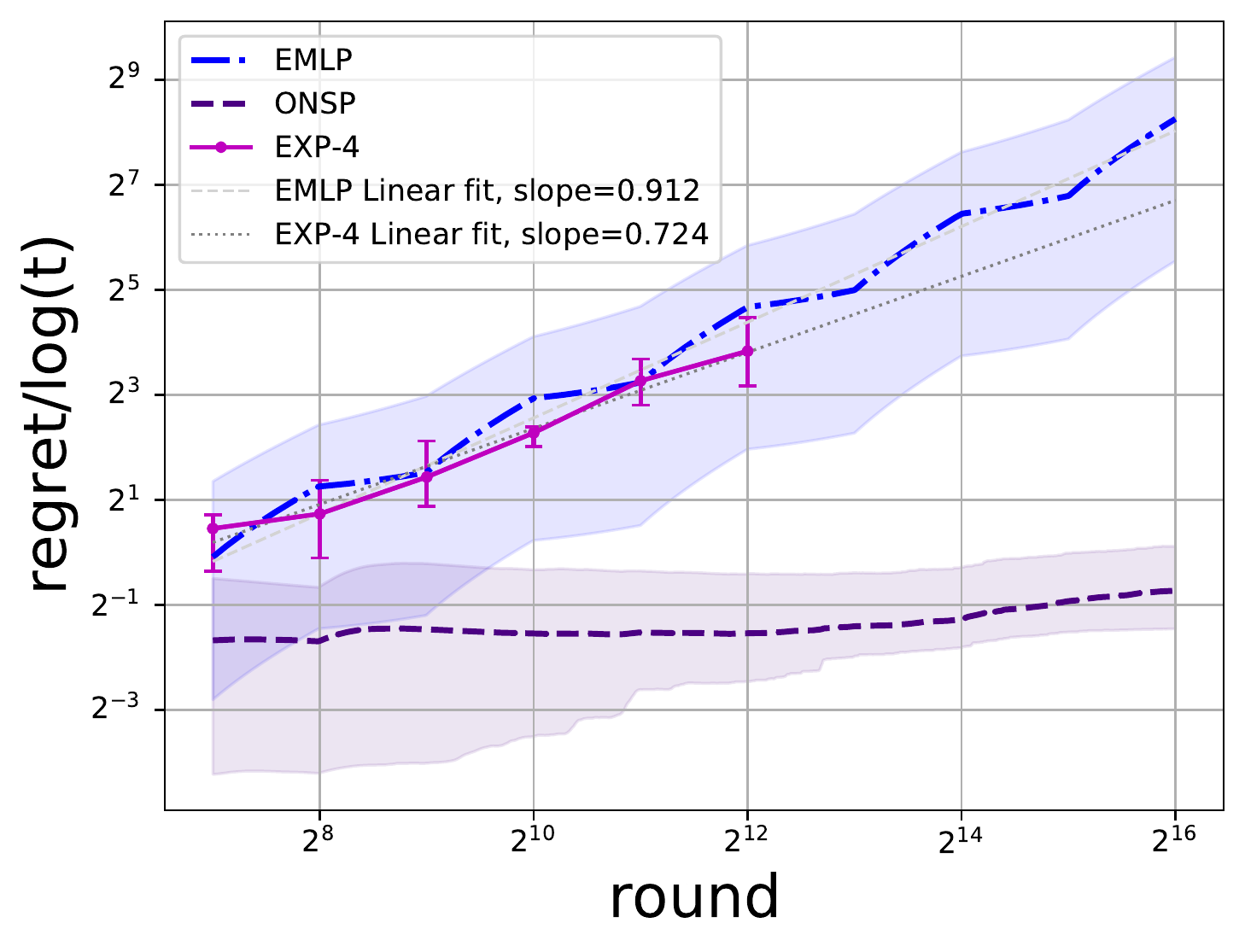}
		\caption{Adversarial feature}\label{adversarial_plot}
	\end{subfigure}
	\caption{\small
		The regret of EMLP, ONSP and EXP-4 on simulated examples (we only conduct EXP-4 up to $T=2^{12}$ due to its exponential time consuming), with Figure \subref{stochastic_plot} for stochastic features and Figure \subref{adversarial_plot} for adversarial ones. The plots are in log-log scales with all regrets divided by a $\log(t)$ factor to show the convergence. For EXP-4, we discretize the parameter space with $T^{-\frac13}$-size grids, which would incur an $\tilde{O}(T^{\frac23})$ regret according to \citet{cohen2020feature_journal}. We also plot linear fits for some regret curves, where a slope-$\alpha$ line indicates an $O(T^{\alpha})$ regret. Besides, we draw error bars and bands with 0.95 coverage using Wald’s test. The two diagrams reveal that (i) logarithmic regrets of EMLP and ONSP in the stochastic setting, (ii) a nearly-linear regret of EMLP in the adversarial setting, and (iii) $O(T^{\frac23})$ regrets of EXP-4 in both settings.
	} \label{fig:plots}
\end{figure*}

In this section, we conduct numerical experiments to validate EMLP and ONSP. In comparison with the existing work, we implement a discretized EXP-4 \citep{auer2002nonstochastic} algorithm for pricing, as is introduced in \citet{cohen2020feature_journal} (in a slightly different setting). We will test these three algorithms in both stochastic and adversarial settings.

Basically, we assume $d=2, B_1=B_2=B=1$ and $N_t\sim\cN(0,\sigma^2)$ with $\sigma = 0.25$. In both settings, we conduct EMLP and ONSP for $T=2^{16}$ rounds. For ONSP, we empirically select $\gamma$ and $\epsilon$ that accelerates the convergence, instead of using the values specified in Lemma \ref{thmons}. Since EXP-4 consumes exponential time and requires the knowledge of $T$ in advance to discretize the policy and valuation spaces, we execute EXP-4 for a series of $T=2^k, k=1,2,\ldots, 12$. We repeat every experiment 5 times for each setting and then take an average.

\ifbool{compact}{\noindent\textbf{Stochastic Setting.}}{\paragraph{Stochastic Setting.}} We implement and test EMLP, ONSP and EXP-4 with stochastic $\{x_t\}$'s. The numerical results are shown in Figure \ref{stochastic_plot} on a log-log diagram, with the regrets divided by $\log(t)$. It shows $\log(t)$-convergences on EMLP and ONSP, while EXP-4 is in a $t^{\alpha}$ rate with $\alpha\approx 0.699$.

\ifbool{compact}{\noindent\textbf{Adversarial Setting.}}{\paragraph{Adversarial Setting.}} We implement the three algorithms and test them with an adversarial $\{x_t\}$'s: for the $k$-th epoch, i.e. $t=2^{k-1}, 2^{k-1}+1, \ldots, 2^k-1$, we let $x_t = [1, 0]^{\top}$ if $k\equiv1(\mod2)$ and $x_t=[0,1]^{\top}$ if $k\equiv0(\mod2)$. The numerical results are shown in Figure \ref{adversarial_plot} on a log-log diagram, with the regrets divided by $\log(t)$. The log-log plots of ONSP and EXP-4 are almost the same as those in Figure \ref{stochastic_plot}. However, EMLP shows an almost linear ($t^{\alpha}$ rate with $\alpha\approx 0.912$) regret in this adversarial setting. This is because the adversarial series only trains one dimension of $\theta$ in each epoch, while the other is arbitrarily initialized and does not necessarily converge. However, in the next epoch, the incorrect dimension is exploited. Therefore, a linear regret originates.

\vspace{-0.5em}
\section{Discussion}
\vspace{-0.5em}
\label{sec_discussion}
Here we discuss the coefficients on our regret bounds as a potential extension of future works. In Appendix \ref{appendix_sec_more_discussion} we will discuss more on algorithmic design, problem modeling, and ethic issues.

\ifbool{compact}{\noindent\textbf{Coefficients on Regret Bounds.}}{\paragraph{Coefficients on Regret Bounds.}}
The exact regret bounds of both EMLP and ONSP contain a constant $\frac{C_{\text{exp}}}{C_{\text{down}}}$ that highly depends on the noise CDF $F$ and could be large. A detailed analysis in Appendix \ref{appendix_coefficient} shows that $\frac{C_{\text{exp}}}{C_{\text{down}}}$ is exponentially large w.r.t. $\frac{B}{\sigma}$ (see Equation \ref{equ_down_exp} and Lemma \ref{lemma_lambda}) for Gaussian noise $\cN(0, \sigma^2)$, which implies that a smaller noise variance would lead to a (much) larger regret bound. This is very counter-intuitive as a larger noise usually leads to a more sophisticated situation, but similar phenomenons also occur in existing algorithms that are suitable for constant-variance noise, such as RMLP in \citet{javanmard2019dynamic} and OORMLP in \citet{wang2020online}. In fact, it is because a (constantly) large noise would help explore the unknown parameter $\theta^*$ and smoothen the expected regret. In this work, this can be addressed by increasing $T$ since we mainly concern the asymptotic regrets as $T\rightarrow\infty$ with fixed noise distributions. However, we admit that it is indeed a nontrivial issue for finite $T$ and small $\sigma$ situations. There exists a ``ShallowPricing'' method in \citet{cohen2020feature_journal} that can deal with a very-small-variance noise setting (when $\sigma=\tilde{O}(\frac1{T})$) and achieve a logarithmic regret. Specifically, its regret bound would decrease as the noise variance $\sigma$ decreases (but would still not reach $O(\log\log{T})$ as the noise vanishes). We might also apply this method as a preprocess to cut the parameter domain and decrease $\frac{B}{\sigma}$ within logarithmic trials (see \citet{cohen2020feature_journal} Thm. 3), but it is still open whether a $\log(T)$ regret is achievable when $\sigma=\Theta(T^{-\alpha})$ for $\alpha\in(0,1)$.
%The regret bounds of EMLP and ONSP involve constants such as $\frac{C_{\text{exp}}}{C_{\text{down}}}$ that highly depends on $F$ and could be large. A detailed analysis in Appendix \ref{appendix_coefficient} shows that the coefficient will increase as the noise distribution gets more concentrated, which seems counter-intuitive. This is actually because a large noise would help explore the unknown parameter $\theta$ and smoothen the expected regret, and can be addressed as follows: On the one hand, we mainly concern the asymptotic regrets as $T\rightarrow\infty$ with fixed noise distributions. On the other hand, if we assume a finite $T$ and a small $\sigma$, we may apply a ``ShallowPricing'' method as a preprocess to cut the parameter domain and decrease $\frac{B}{\sigma}$ within logarithmic trials (see \citet{cohen2020feature_journal} Thm. 3). However, it is still open whether a $\log(T)$ regret is achievable when $\sigma=\Theta(T^{-\alpha})$ for $\alpha\in(0,1)$.

\section{Conclusion}
\label{sec_conclusion}
In this work, we studied the problem of  online feature-based dynamic pricing with a noisy linear valuation in both stochastic and adversarial settings. We proposed a max-likelihood-estimate-based algorithm (EMLP) for stochastic features and an online-Newton-step-based algorithm (ONSP) for adversarial features. Both of them enjoy a regret guarantee of $O(d\log{T})$, which also attains the information-theoretic limit up to a constant factor. Compared with existing works, EMLP gets rid of strong assumptions on the distribution of the feature vectors in the stochastic setting, and ONSP improves the regret bound exponentially from $O(T^{2/3})$ to $O(\log{T})$ in the adversarial setting. We also showed that knowing the noise distribution (or the demand curve) is required to obtain logarithmic regret, where we prove a  lower bound of $\Omega(\sqrt{T})$ on the regret for the case when the noise is knowingly Gaussian but with an  unknown $\sigma$. In addition, we conducted numerical experiments to empirically validate the scaling of our algorithms. Finally, we discussed the regret dependence on the noise variance, and proposed a subtle open problem for further study.% Finally, we have discussed our results from the perspective of the real-life selling process and propose some open problems necessary for further study.

%\bibliographystyle{abbrvnat}  % apa-good
%\bibliography{ref_log}

\section*{Acknowledgments}
The work is partially supported by the Adobe Data Science Award and a start-up grant from the UCSB Department of Computer Science. We appreciate the input from anonymous reviewers and AC as well as a discussion with Akshay Krishnamurthy for clarifying some details of \citet{krishnamurthy2020contextual}.

\bibliographystyle{abbrvnat}  % apa-good
\bibliography{ref_log}

%%%%%%%%%%%%%%%%%%%%%%%%% Checklist %%%%%%%%%%%%%%%%%%%%%%%%
% \input{log_checklist_neurips21}
%%%%%%%%%%%%%%%%%%%%%%%%%%%%%%%%%%%%%%%%%%%%%%%%%%%%%%%%%%%%
\newpage

\appendix

\onecolumn
\renewcommand\thesection{\Alph{section}}
\renewcommand\thesubsection{\Alph{section}.\arabic{subsection}}

\begin{appendices}
\textbf{\huge APPENDIX}

\section{Other related works}\label{appendix_other_related_works}
Here we will briefly review the history and recent studies that are related to our work. For the historical introductions, we mainly refer to \citet{den2015dynamic} as a survey. For bandit approaches, we will review some works that apply bandit algorithms to settle pricing problems. For the structural models, we will introduce different modules based on the review in \citet{chan2009structural}. Based on the existing works, we might have a better view of our problem setting and methodology.

\subsection{History of Pricing}

{It was the work of \citet{cournot1897researches} in 1897 that firstly applied mathematics to analyze the relationship between prices and demands. In that work, the price was denoted as $p$ and the demand was defined as a \emph{demand function} $F(p)$. Therefore, the revenue could be written as $pF(p)$. This was a straightforward interpretation of the general pricing problem, and the key to solving it was estimations of $F(p)$ regarding different products. Later in 1938, the work \citet{schultz1938theory} proposed price-demand measurements on exclusive kinds of products.  It is worth mentioning that these problems are ``static pricing’’ ones, because $F$ is totally determined by price $p$ and we only need to insist on the optimal one to maximize our profits.%\citep{}
	
However, the static settings were qualified by the following two observations: on the one hand, a demand function may not only depends on the static value of $p$, but also be affected by the trend of $p$'s changing \citep{evans1924dynamics, mazumdar2005reference}; on the other hand, even if $F(p)$ is static, $p$ itself might change over time according to other factors such as inventory level \citep{kincaid1963inventory}. As a result, it is necessary to consider dynamics in both demand and price, which leads to a ``dynamic pricing’’ problem setting.}

\subsection{Dynamic Pricing as Bandits}
\label{d_p_a_bandits}
As is said in Section \ref{sec_related_works}, the pricing problem can be viewed as a stochastic contextual bandits problem \citep[see, e.g.,][]{langford2007epoch,agarwal2014taming}. Even though we may not know the form of the demand function, we can definitely see feedback of demands, i.e. how many products are sold out, which enables us to learn a better decision-making policy. Therefore, it can be studied in a bandit module. If the demand function is totally agnostic, i.e. the evaluations (the highest prices that customers would accept) come at random or even at adversary over time, then it can be modeled as a Multi-arm bandit (MAB) problem \citep{whittle1980multi} exactly. In our paper, instead, we focus on selling different products with a great variety of features. This can be characterized as a Contextual bandit (CB) problem \citep{auer2002nonstochastic, langford2007epoch}. The work \citet{cohen2020feature_journal}, which applies the ``EXP-4'' algorithm from \citet{auer2002nonstochastic}, also mentions that ``the arms represent prices and the payoffs from the different arms are correlated since the measures of demand evaluated at different price points are correlated random variables’’. A variety of existing works, including \citet{ kleinberg2003value, araman2009dynamic, chen2013simple, keskin2014dynamic, besbes2015surprising}, has been approaching the demand function from a perspective of from either parameterized or non-parameterized bandits.

However, our problem setting is different from a contextual bandits setting in at least two perspectives: feedback and regret. The pricing problem has a specially structured feedback between full information and bandits setting. Specifically, $r_t>0$ implies that all policies producing $v<v_t$ will end up receiving $r'_t=v$, and $r_t =0$ implies that all policies producing $v>v_t$ will end up receiving $r'_t=0$. However, the missing patterns are confounded with the rewards. Therefore it is non-trivial to leverage this structure to improve the importance sampling approach underlying the algorithm of \citet{agarwal2014taming}. We instead consider the natural analog to the linear contextual bandits setting \citep{chu2011contextual}\footnote{But do notice that our expected reward above is not linear, even if the valuation function is.} and demonstrate that in this case an exponential improvement in the regret is possible using the additional information from the censored feedback. As for regret, while in contextual bandits it refers to a comparison with the optimal policy, it is here referring to a comparison with the optimal \emph{action}. In other words, though our approaches (both in EMLP and in ONSP) are finding the true parameter $\theta^{*}$, the regret is defined as the ``revenue gap'' between the optimal price and our proposed prices. These are actually equivalent in our fully-parametric setting (where we assume a linear-valuation-known-noise model), but will differ a lot in partially parametric and totally agnostic settings.

%Dynamic pricing is a bandit problem as we consider prices as actions. % Therefore, it can be studied in a bandit module.%can be treated in methods of machine learning, as we learn the demand function based on existing results and optimize the estimated revenue to determine an ``optimal’’ price. Since we will not have full information on the demand function, there exists a balance between explorations and exploitations, i.e. trying higher prices with risks or adopting lower ones for safety.
%\citep{}

\subsection{Structural Model}
{While a totally agnostic model guarantees the most generality, a structural model would help us better understand the mechanism behind the observation of prices and demands. The key to a structural pricing model is the \emph{behavior} of agents in the market, including customers and/or firms. In other words, the behavior of each side can be described as a decision model. 
From the perspective of demand (customers), the work \citet{kadiyali1996empirical} adopts a linear model on laundry detergents market, \citet{iyengar2007model} and \citet{lambrecht2007does} study three-part-tariff pricing problems on wireless and internet services with mixed logit models. Besanko et al. assumed an aggregate logit model on customers in works \citet{besanko1998logit} and \citet{besanko2003competitive} in order to study the competitive behavior of manufacturers in ketchup market. Meanwhile, the supply side is usually assumed to be more strategic, such as Bertrand-Nash behaviors \citep{kadiyali1996empirical, besanko1998logit, draganska2006consumer}. For more details, please see \citet{chan2009structural}.}

\section{Proofs}
\subsection{Proof of Lemma \ref{g_unimodal}}
\label{appendix_subsec_proof_lemma_g_unimodal}
\begin{proof}
	Since $v^{*} = \argmax g(v,u)$, we have:
	\begin{equation*}
	\begin{aligned}
	&\frac{\partial g(v,u)}{\partial v}|_{v=v^*}=0\\
	\Leftrightarrow& 1-F(v^{*}-u)-v^{*}\cdot f(v^{*}-u)=0\\
	\Leftrightarrow& \frac{1-F(v^{*}-u)}{f(v^{*}-u)}-(v^{*}-u)=u\\
	\end{aligned}
	\end{equation*}
	Define $\varphi(\omega) = \frac{1-F(\omega)}{f(\omega)}-\omega$, and we take derivatives:
	\begin{equation*}
	\varphi'(\omega)=\frac{-f^2(\omega)-(1-F(\omega))f'(w)}{f^2(w)}-1=\frac{\text{d}^2\log(1-F(\omega))}{\text{d}\omega^2}\cdot\frac{(1-F(\omega))^2}{(f(\omega))^2}-1<-1,
	\end{equation*}
	where the last equality comes from the strict log-concavity of $(1-F(\omega))$. Therefore, $\varphi(\omega)$ is decreasing and $\varphi(+\infty)=-\infty$. Also, notice $\varphi(-\infty)=+\infty$, we know that for any $u\in\R$, there exists an $\omega$ such that $\varphi(\omega)=u$. For $u\geq 0$, we know that $g(v,u)\geq0$ for $v\geq0 $ and $g(v,u)<0$ for $v<0$. Therefore, $v^{*}\geq 0$ if $u\geq 0$.
\end{proof}
\begin{comment}
\begin{proof}
	
			First of all, we define a function $h(v,u):=\frac{\partial}{\partial v} g(v,u)$. Therefore, we know that:
			\begin{equation}
			\begin{aligned}
			h(v,u)=&\Phi(u-v)-v\cdot p(u-v),\\
			\frac{\partial}{\partial v}h(v,u)=&\frac{1}{\sigma^2}(v^2-uv-2\sigma^2)p(u-v).
			\end{aligned}
			\end{equation}
			Suppose $v_0>0$ satisfies

			\begin{equation}
			v_0^2-uv_0-2\sigma^2=0.
			\end{equation}
			Therefore, 
			\begin{equation}
			\begin{aligned}
			\frac{\partial}{\partial v}&h(v,u)<0, v\in(0,v_0),\\
			\frac{\partial}{\partial v}&h(v,u)>0, v\in(v_0,+\infty).\\
			\end{aligned}
			\end{equation}
			This indicates that $h(v,u)$ monotonically decreases w.r.t. $v$ in $(0,v_0)$, and monotonically increases in $(v_0,+\infty)$. Also, we know the following facts:
			\begin{equation}
			\begin{aligned}
			h(0,u)=\Phi(u)>0;\\
			\lim_{v\rightarrow+\infty}h(v,u)=0.
			\end{aligned}
			\end{equation}
			Therefore, we know that $h(v_0,u)<0$. Since $h(v_0,u)$ decreases w.r.t $v$ in $(0,v_0)$, and $h(0,u)>0, h(v_0,u)<0$, we have:
			
			\begin{equation}
			\exists v_1\in(0,v_0),\ s.t.\ h(v_1,u)=0.
			\end{equation}
			As a result,
			\begin{equation}
			\begin{aligned}
			h(v,u)>&0, v\in(0,v_1),\\
			h(v,u)<&0, v\in(v_1,+\infty).\\
			\end{aligned}
			\end{equation}
			This indicates that $g(v,u)$ increases w.r.t. $v$ in $(0,v_1)$, and decreases in $(v_1,+\infty)$. Thus the lemma holds.
			
\end{proof}
\end{comment}

\subsection{Proofs in Section \ref{subsec_random_analysis}}

%\subsubsection{Proof of Theorem \ref{corollary_expected regret bound} from Theorem \ref{expected regret bound}}

%\label{proof_corollary_expected_regret_bound}

%\input{proof_random/proof_of_corollary.tex}

\subsubsection{Proof of Lemma \ref{analytical_regret_bound}}

\label{appendix_subsec_proof_of_regret_bound}

\begin{proof}
		We again define $\varphi(\omega) = \frac{1-F(\omega)}{f(\omega)}-\omega$ as in Appendix \ref{appendix_subsec_proof_lemma_g_unimodal}. According to Equation \ref{def_J}, we have:
		\begin{equation}
		\begin{aligned}
			\ &\frac{\partial g(v,u)}{\partial v}|_{v=J(u)}=0\\
			\Rightarrow& 1-F(J(u)-u)-J(u)\cdot f(J(u)-u)=0\\
			\Rightarrow& \varphi(J(u)-u)=u\\
			\Rightarrow& J(u)=u+\varphi^{-1}(u)\\
			\Rightarrow& J'(u) =  1+\frac{1}{\varphi'(\varphi^{-1}(u))}.
		\end{aligned}
		\label{eqn_implicit_derivatives}
		\end{equation}
		The last line of Equation \ref{eqn_implicit_derivatives} is due to the Implicit Function Derivatives Principle. From the result in Appendix \ref{appendix_subsec_proof_lemma_g_unimodal}, we know that $\varphi'(\omega)<-1, \forall\omega\in\R$. Therefore, we have $J'(u)\in(0,1), u\in\R$, and hence $0 \geq J(u)<u+J(0)$ for $u\geq 0$. Since $u\in[0,B]$, we may assume that $v\in[0,B+J(0)]$ without losing generality. In the following part, we will frequently use this range.
		
		Denote $u:=x_t^{\top}\theta, u^{*}=x_t^{\top}\theta^{*}$. According to Equation \ref{eqn_parameter_based_regret}, we know that:
		\begin{equation*}
		\begin{aligned}
			Reg_t(\theta)&=g(J(u^{*}),u^{*})-g(J(u),u^{*})\\
			&=-\frac{\partial g(v,u^{*})}{\partial v}|_{v=J(u^{*})}(J(u^{*}-J(u)))+\frac12\left(-\frac{\partial^2 g(v,u^{*})}{\partial v^2}|_{v=\tilde{v}}\right)(J(u^{*})-J(u))^2\\
			&\leq 0 + \frac12\max_{\tilde{v}\in[0, B+J(0)]} \left(-\frac{\partial^2 g(v,u^{*})}{\partial v^2}|_{v=\tilde{v}}\right)\cdot(J(u^{*})-J(u))^2\\
			&=\frac12\max_{\tilde{v}\in[0, B+J(0)]} \left(2f(\tilde{v}-u^{*})+\tilde{v}\cdot{f'}(\tilde{v}-u^{*})\right)\cdot(J(u^{*})-J(u))^2\\
			&\leq\frac12(2B_f+(B+J(0))\cdot B_{f'})(J(u^{*})-J(u))^2\\
			&\leq\frac12(2B_f+(B+J(0))\cdot B_{f'})(u^{*}-u)^2\\
			&=\frac12(2B_f+(B+J(0))\cdot B_{f'})(\theta^{*}-\theta)^{\top}x_tx_t^{\top}(\theta^{*}-\theta).
		\end{aligned}
		\end{equation*}
		Here the first line is from the definition of $g$ and $Reg(\theta)$, the second line is due to Taylor's Expansion, the third line is from the fact that $J(u^{*})$ maximizes $g(v,u^{*})$ with respect to $v$, the fourth line is by calculus, the fifth line is from the assumption that $0<f(\omega)\leq B_{f}, |f'(\omega)|\leq B_{f'}$ and $v\in[0, B+J(0)]$, the sixth line is because of $J'(u)\in(0,1), \forall u\in\R$, and the seventh line is from the definition of $u$ and $u^{*}$.
\end{proof}

\subsubsection{Proof of Lemma \ref{strong_convexity}}

\label{appendix_subsec_proof_of_strong_convexity}

\begin{proof}
	We take derivatives of $l_{t}(\theta)$, and we get:
	
	\begin{equation}
	\begin{aligned}
		l_{t}(\theta)=&\ind_t\left(-\log(1-F(v_t-x_t^{\top}\theta))\right)+(1-\ind_t)\left(-\log(F(v_t-x_t^{\top}\theta))\right)\\
		\nabla l_t(\theta)  =&\ind_t\left(-\frac{f(v_t-x_t^{\top}\theta)}{1-F(v_t-x_t^{\top}\theta)}\right)\cdot x_t+(1-\ind_t)\left(\frac{f(v_t-x_t^{\top}\theta)}{F(v_t-x_t^{\top}\theta)}\right)\cdot x_t\\
		\nabla^2 l_t(\theta) =&\ind_t\cdot\frac{f(v_t-x_t^{\top}\theta)^2+f'(v_t-x_t^{\top}\theta)\cdot(1-F(v_t-x_t^{\top}\theta))}{(1-F(v_t-x_t^{\top}\theta))^2}\cdot{x_tx_t^{\top}}\\
		&+ (1-\ind_t)\cdot\frac{f(v_t-x_t^{\top}\theta)^2-f'(v_t-x_t^{\top}\theta)F(v_t-x_t^{\top}\theta)}{F(v_t-x_t^{\top}\theta)^2}\cdot{x_tx_t^{\top}}\\
		=&\ind_t\cdot\frac{-\text{d}^2 \log(1-F(\omega))}{\text{d}\omega^2}|_{\omega=v_t-x_t^{\top}\theta}\cdot{x_tx_t^{\top}} + (1-\ind_t)\frac{-\text{d}^2 \log(F(\omega))}{\text{d}\omega^2}|_{\omega=v_t-x_t^{\top}\theta}\cdot{x_tx_t^{\top}}\\
		\succeq&\inf_{\omega\in[-B, B+J(0)]}\min\left\{\frac{\text{d}^2\log(1-F(\omega))}{\text{d}\omega^2}, \frac{\text{d}^2\log(F(\omega))}{\text{d}\omega^2}\right\}\\
		=&C_{\text{down}}x_tx_t^{\top},
	\end{aligned}
	\label{l_grad_Hess}
	\end{equation}
	which directly proves the first inequality. For the second inequality, just notice that
	\begin{equation}
	\begin{aligned}
		\nabla l_t(\theta)\nabla l_t(\theta)^{\top}=&\ind_t\left(\frac{f(v_t-x_t^{\top}\theta)}{1-F(v_t-x_t^{\top}\theta)}\right)^2x_tx_t^{\top} + (1-\ind_t)\left(\frac{f(v_t-x_t^{\top}\theta)}{F(v_t-x_t^{\top}\theta)}\right)^2x_tx_t^{\top}\\
		\preceq&\sup_{\omega\in[-B, B+J(0)]}\max\{\left(\frac{f(\omega)}{F(\omega)}\right)^2, \left(\frac{f(\omega)}{1-F(\omega)}\right)^2\}x_tx_t^{\top}\\
		=&C_{\text{exp}}x_tx_t^{\top}.
	\end{aligned}
	\end{equation}
	The only thing to point out is that $\frac{f(\omega)}{F(\omega)}$ and $\frac{f(\omega)}{1-F(\omega)}$ are all continuous for $\omega\in[-B, B+J(0)]$, as $F(\omega)$ is strictly increasing and thus $0<F(\omega)<1, \omega\in\R$.
\end{proof}

\subsubsection{Proof of Lemma \ref{before_expectation}}

\label{appendix_subsec_proof_of_uniform_stability}

		\begin{proof}
			{
				In the following part, we consider a situation that an epoch of $n\geq 2$ rounds of pricing is conducted, generating $l_j(\theta)$ as negative likelihood functions, $j=1,2,\ldots, n$. Define a ``\textbf{leave-one-out}''negative log-likelihood function $$\tilde{L}_i(\theta)=\frac{1}{n}\sum_{j=1,j\neq i}^{n}l_j(\theta),$$ and let $$\tilde{\theta}_{i}:=\mathop{\arg\min}\limits_{\theta}\tilde{L}_{i}(\theta).$$ Based on this definition, we know that $\tilde{\theta}_{i}$ is independent to $l_i(\theta)$ given historical data, and that $\tilde{\theta}_{i}$ are identically distributed for all $i=1,2,3,\ldots, n$.%It is not hard to verify that $\tilde{\theta}_{i}$ are i.i.d. random variables given historical data (i.e. given the estimated parameter for pricing purpose).
				
				In the following part, we will firstly propose and proof the following inequality:
				
				\begin{equation}
				\frac{1}{n}\sum_{i=1}^{n}(l_i(\tilde{\theta}_i)-l_i(\hat{\theta}))\leq\frac{C_{\text{exp}}}{C_{down}}\frac{d}{n}=O(\frac{d}{n}),
				\label{leave_one_out_inequality}
				\end{equation}

				where $\hat{\theta}$ is the short-hand notation of $\hat{\theta}_k$ as we do not specify the epoch $k$ in this part. We now cite a lemma from \citet{koren2015fast}:
				\begin{lemma}
					Let $g_1$, $g_2$ be 2 convex function defined over a closed and convex domain $\mathcal{K}\subseteq\mathbb{R}^{d}$, and let $x_1=\arg\min_{x\in\mathcal{K}}g_1(x)$ and $x_2=\arg\min_{x\in\mathcal{K}}g_2(x)$. Assume $g_2$ is locally $\delta$-strongly-convex at $x_1$ with respect to a norm $||\cdot||$. Then, for $h=g_2-g_1$ we have
					$$||x_2-x_1||\leq\frac{2}{\delta}||\nabla h(x_1)||_*.$$
					Here $||\cdot||_*$ denotes a dual norm.
					\label{lemma_dual_norm_of_gradient_convex}
				\end{lemma}
				The following is a proof of this lemma.
				\begin{proof} (of Lemma \ref{lemma_dual_norm_of_gradient_convex})
					According to convexity of $g_2$, we have:
					\begin{equation}
					g_2(x_1)\geq g_2(x_2)+\nabla g_2(x_2)^{\top}(x_1-x_2).
					\label{g_2_convex}
					\end{equation}
					According to strong convexity of $g_2$ at $x_1$, we have:
					\begin{equation}
					g_2(x_2)\geq g_2(x_1)+\nabla g_2(x_1)^{\top}(x_2-x_1)+\frac{\delta}{2}||x_2-x_1||^2.
					\label{g_2_strongly_convex}
					\end{equation}
					Add Equation \eqref{g_2_convex} and \eqref{g_2_strongly_convex}, and we have:
					\begin{equation}
					\begin{aligned}
					&g_2(x_1)+g_2(x_2)\geq g_2(x_2)+g_2(x_1)+(\nabla g_2(x_1)-\nabla g_2(x_2))^{\top}(x_2-x_1)+\frac{\delta}{2}||x_2-x_1||^2\\
					\Leftrightarrow\qquad &(\nabla g_2(x_1)-\nabla g_2(x_2))^{\top}(x_1-x_2)\geq\frac{\delta}{2}||x_1-x_2||^2\\
					\Leftrightarrow\qquad &(\nabla g_1(x_1)+\nabla h(x_1)-\nabla g_2(x_2))^{\top}(x_1-x_2)\geq\frac{\delta}{2}||x_1-x_2||^2\\
					\Leftrightarrow\qquad &\nabla h(x_1)^{\top}(x_1-x_2)\geq \frac{\delta}{2}||x_1-x_2||^2\\
					\Rightarrow\qquad&||\nabla h(x_1)||_{*}||x_1-x_2||\geq\frac{\delta}{2}||x_1-x_2||^2\\
					\Rightarrow\qquad&||\nabla h(x_1)||_{*}\geq\frac{\delta}{2}||x_1-x_2||.
					\end{aligned}
					\end{equation}
					The first step is trivial. The second step is a sequence of $g_2=g_1+h$. The third step is derived by the following 2 first-order optimality conditions: $\nabla g_1(x_1)^{\top}(x_1-x_2)\leq0$, and $\nabla g_2(x_2)^{\top}(x_2-x_1)\leq0$. The fourth step is derived from Holder's Inequality:
					\begin{equation*}
					||\nabla h(x_1)||_{*}||x_1-x_2||\geq\nabla h(x_1)^{\top}(x_1-x_2).
					\end{equation*}
					Therefore, the lemma holds.
				\end{proof}
				
				In the following part, we will set up a strongly convex function of $g_2$. Denote $H=\sum_{t=1}^{n}x_{t}x_{t}^{\top}$. From Lemma \ref{strong_convexity}, we know that
				$$\nabla^2\hat{L}(\theta)\succeq C_{down}\frac{1}{n}H.$$ Here $\hat{L}(\theta)$ is the short-hand notation of $\hat{L}_{k}(\theta)$ as we do not specify $k$ in this part. Since we do not know if $H$ is invertible, i.e. if a norm can be induced by $H$, we cannot let $g_2(\theta)=\hat{L}(\theta)$. Instead, we change the variable as follows:
				
				We first apply singular value decomposition to $H$, i.e. $H=U\Sigma U^{\top}$, where $U\in\mathbb{R}^{d\times r}, U^{\top}U=I_r, \Sigma=diag\{\lambda_1, \lambda_2, \ldots, \lambda_r\}\succ 0 $. After that, we introduce a new variable $\eta:=U^{\top}\theta$. Therefore, we have $\theta=U\eta+V\epsilon$, where $V\in\mathbb{R}^{d\times(d-r)}, V^{\top}V=I_{d-r}, V^{\top}U=0$ is the standard orthogonal bases of the null space of $U$, and $\epsilon\in\mathbb{R}^{(d-r)}$. Similarly, we define $\tilde{\eta}_i=U^{\top}\tilde{\theta}_i$ and $\hat{\eta}=U^{\top}\hat{\theta}$. According to these, we define the following functions:
				\begin{equation}
				\begin{aligned}
				f_i(\eta)&:=l_i(\theta)=l_i(U\eta+V\epsilon)\\
				\tilde{F}_i(\eta)&:=\tilde{L}_i(\theta)=\tilde{L}_i(U\eta+V\epsilon)\\
				\hat{F}(\eta)&:=\hat{L}(\theta)=\hat{L}(U\eta+V\epsilon).
				\end{aligned}
				\label{definition_F_and_Fi_and_f}
				\end{equation}
				Now we prove that $\hat{F}(\eta)$ is locally-strongly-convex. Similar to the proof of Lemma \ref{strong_convexity}, we have:
				\begin{equation}
				\begin{aligned}
				\nabla^2\hat{F}(\eta)=&\frac{1}{n}\sum_{i=1}^{n}\nabla^2f_i(\eta)\\
				=&\frac{1}{n}\sum_{i=1}^{n}\frac{\partial^2l_i}{\partial(x_i^{\top}\theta)^2}(\frac{\partial x_i^{\top}\theta}{\partial\eta})(\frac{\partial x_i^{\top}\theta}{\partial\eta})^{\top}\\
				=&\frac{1}{n}\sum_{i=1}^{n}\frac{\partial^2l_i}{\partial(x_i^{\top}\theta)^2}(\frac{\partial x_i^{\top}(U\eta+V\epsilon)}{\partial\eta})(\frac{\partial x_i^{\top}(U\eta+V\epsilon)}{\partial\eta})^{\top}\\
				=&\frac{1}{n}\sum_{i=1}^{n}\frac{\partial^2l_i}{\partial(x_i^{\top}\theta)^2}(U^{\top}x_i)(U^{\top}x_i)^{\top}\\
				\succeq&\frac{1}{n}\sum_{i=1}^{n}C_{down}U^{\top}x_ix_i^{\top}U\\
				=&\frac{1}{n}C_{down}U^{\top}(\sum_{i=1}^{n}x_ix_i^{\top})U^{\top}\\
				=&\frac{1}{n}C_{down}U^{\top}HU\\
				=&\frac{1}{n}C_{down}U^{\top}U\Sigma U^{\top}U\\
				=&\frac{1}{n}C_{down}\Sigma\\
				\succ&0
				\end{aligned}
				\end{equation}
				That is to say, $\hat{F}(\eta)$ is locally $\frac{C_{down}}{n}$-strongly convex w.r.t $\Sigma$ at $\eta$. Similarly, we can verify that $\tilde{F}_i(\eta)$ is convex (not necessarily strongly convex). Therefore, according to Lemma \ref{lemma_dual_norm_of_gradient_convex}, let $g_1(\eta)=\tilde{F}_i(\eta), g_2(\eta)=\hat{F}(\eta)$, and then $x_1=\tilde{\eta}_i=U^{\top}\tilde{\theta}_i$, $x_2=\hat{\eta}=U^{\top}\hat{\theta}$. Therefore, we have:
				\begin{equation}
				||\hat{\eta}-\tilde{\eta}_i||_{\Sigma}\leq\frac{1}{C_{down}}||\nabla f_i(\tilde{\eta}_i)||_{\Sigma}^{*}.
				\end{equation}
				Now let us show the validation of this theorem:
				\begin{equation}
				\begin{aligned}
				l_i(\tilde{\theta}_i)-l_i(\hat{\theta})=&f_i(\tilde{\eta}_i)-f_i(\hat{\eta})\\
				\explain{\leq}{\text{convexity}}&\nabla f_i(\tilde{\eta}_i)^{\top}(\tilde{\eta}_i-\hat{\eta})\\
				\explain{\leq}{\text{Holder\ inequality}}&||\nabla f_i(\tilde{\eta}_i)||^{*}_{\Sigma}||\tilde{\eta}_i-\hat{\eta}||_{\Sigma}\\
				\explain{\leq}{\text{Lemma   \ref{lemma_dual_norm_of_gradient_convex}}}&\frac{1}{C_{down}}(||\nabla f_i(\tilde{\eta}_i)||^{*}_{\Sigma})^2.
				\end{aligned}
				\end{equation}
				
				And thus we have
				\begin{equation}
				\begin{aligned}
				\sum_{i=1}^{n}l_i(\tilde{\theta}_i)-l_i(\hat{\theta})&\leq\frac{1}{C_{down}}\sum_{i=1}^{n}||\nabla f_i(\tilde{\eta}_i)||^{*}_{\Sigma})^2\\
				&\leq\frac{1}{C_{down}}\sum_{i=1}^{n}(\frac{p}{\Phi})_{\max}^2x_i^{\top}U\Sigma^{-1}U^{\top}x_i\\
				&=\frac{C_{\text{exp}}}{C_{down}}C_{\text{exp}}\sum_{i=1}^{n}tr(x_i^{\top}U\Sigma^{-1}U^{\top}x_i)\\
				&=\frac{C_{\text{exp}}}{C_{down}}C_{\text{exp}}\sum_{i=1}^{n}tr(U\Sigma^{-1}U^{\top}x_ix_i^{\top})\\
				&=\frac{C_{\text{exp}}}{C_{down}}C_{\text{exp}}tr(U\Sigma^{-1}U^{\top}\sum_{i=1}^{n}x_ix_i^{\top})\\
				&=\frac{C_{\text{exp}}}{C_{down}}C_{\text{exp}}tr(U\Sigma^{-1}U^{\top}H)\\
				&=\frac{C_{\text{exp}}}{C_{down}}C_{\text{exp}}tr(U\Sigma^{-1}U^{\top}U\Sigma U^{\top})\\
				&=\frac{C_{\text{exp}}}{C_{down}}C_{\text{exp}}tr(UU^{\top})\\
				&=\frac{C_{\text{exp}}}{C_{down}}C_{\text{exp}}tr(U^{\top}U)\\
				&=\frac{C_{\text{exp}}}{C_{down}}C_{\text{exp}}tr(I_r)\\
				&=\frac{C_{\text{exp}}}{C_{down}}C_{\text{exp}}r\\
				&\leq\frac{C_{\text{exp}}}{C_{down}}d.
				\end{aligned}
				\label{c_1_location}
				\end{equation}
				Thus the Inequality \ref{leave_one_out_inequality} is proved. After that, we have:
				
				 \begin{equation*}
				\begin{aligned}
				&\mathbb{E}_{h}[L(\tilde{\theta}_{n})]-L(\theta^{*})\\
				=&\mathbb{E}_{h}[L(\tilde{\theta}_{n})]-\mathbb{E}_{h}[\hat{L}(\theta^{*})]\\
				\leq&\mathbb{E}_{h}[L(\tilde{\theta}_{n})]-\mathbb{E}_{h}[\hat{L}(\hat{\theta})]\\
				=&\frac{1}{n}\sum_{i=1}^{n}\mathbb{E}_{h}[L(\tilde{\theta}_{i})]-\mathbb{E}_{h}[\hat{L}(\hat{\theta})]\\
				=&\frac{1}{n}\sum_{i=1}^{n}\mathbb{E}_{h}[l_i(\tilde{\theta}_{i})]-\mathbb{E}_{h}[\hat{L}(\hat{\theta})]\\
				=&\frac{1}{n}\sum_{i=1}^{n}\mathbb{E}_{h}[l_i(\tilde{\theta}_{i})]-\sum_{i=1}^{n}\mathbb{E}_{h}[l_i(\hat{\theta})]\\
				=&\frac{1}{n}\sum_{i=1}^{n}\mathbb{E}_{h}[l_i(\tilde{\theta}_{i})-l_i(\hat{\theta})]\\
				\leq&\frac{C_{\text{exp}}}{C_{down}}\frac{d}{n}\\
				=&O(\frac{d}{n})
				\end{aligned}
				\end{equation*}
				Thus we has proved that $\mathbb{E}_{h}[L(\tilde{\theta}_{n})]-L(\theta^{*})\leq\frac{C_{\text{exp}}}{C_{down}}\cdot\frac{d}{n}$. Notice that $\tilde{\theta}_{n}$ is generated by optimizing the leave-one-out likelihood function $\tilde{L}_n(\theta)=\sum_{j=1}^{n-1}l_j(\theta)$, which does not contain $l_{n}(\theta)$, and that the expected likelihood function $L(\theta)$ does not depend on any specific result occurring in this round. That is to say, every term of this inequality is not related to the last round $(x_{n}, v_{n}, \mathbbm{1}_{n})$ at all. In other words, this inequality is still valid if we only conduct this epoch from round $1$ to $(n-1)$.
				
				Now let $n=\tau+1$, and then we know that $\tilde{\theta}_{\tau+1}=\hat{\theta}$. Therefore, the theorem holds.
				%Now we change all $\tau$ to $(\tau+1)$. Since we know that $\hat{\theta}=\tilde{\theta_{\tau+1}}$, and then the theorem holds.
			}
		\end{proof}

%\subsection{Proofs in Section \ref{subsec_adversarial_analysis}}

% \subsubsection{Proof of Theorem \ref{thmupb}}\label{adv_append_proof_of_surrogate_loss}
% \input{proof_adversarial/proof_surrogate_loss.tex}

\subsection{Proof of Lower bound in Section \ref{sec_lower}}
\label{appendix_proof_lower_bound}
\begin{proof}
    We assume a fixed $u^{*}$ such that $x^{\top}\theta^{*}=u^{*}, \forall x\in\mathbb{D}$. In other words, we are considering a non-context setting. Therefore, we can define a policy as $\Psi:\{0,1\}^t\rightarrow\R^{+}, t=1,2,\ldots$ that does not observe $x_t$ at all. Before the proof begins, we firstly define a few notations: We denote $\Phi_{\sigma}(\omega)$ and $p_{\sigma}(\omega)$ as the CDF and PDF of Gaussian distribution $\cN(0,\sigma^2)$, and the corresponding $J_{\sigma}(u)=\arg\max_{v}v(1-\Phi_{\sigma}(v-u))$ as the pricing function.

    Since we have proved that $J'(u)\in(0,1)$ for $u\in\R$ in Appendix \ref{appendix_subsec_proof_of_strong_convexity}, we have the following lemma:
    \begin{lemma}\label{lemma_j_prime}
        $u-J_{\sigma}(u)$ monotonically increases as $u\in(0,+\infty), \forall \sigma>0$. Also, we know that $J_{\sigma}(0)>0, \forall\sigma>0$.
    \end{lemma}
    	Now consider the following cases: $\sigma_1=1, \sigma_2=1-f(T)$, where $\lim_{T\rightarrow\infty}f(T)=0, f'(T)<0, 0<f(T)<\frac{1}{2}$. We will later determine the explicit form of $f(T)$.
	
	Suppose $u^{*}$ satisfies $J_{\sigma_{1}}(u^{*})=u^{*}$. Solve it and get $u^{*}=\sqrt{\frac{\pi}{2}}$. Therefore, we have $u\in(0,u^{*})\Leftrightarrow J_{1}(u)>u$, and $u\in(u^{*},+\infty)\Leftrightarrow J_{1}(u)<u$. As a result, we have the following lemma.%Therefore, if we fix the value of ALL products at $u*$, then the problem model is only parameterized by $\sigma$. Since $\sigma_2\in(\frac{1}{2},1)$, we have the following lemma:
	
	\begin{lemma}
		For any $\sigma\in(\frac{1}{2}, 1)$, we have:
		\begin{equation}
		J_{\sigma}(u^{*})\in(0,u^{*})
		\end{equation}
	\end{lemma}

	\begin{proof}
		Firstly, we have:
		\begin{equation*}
		\begin{aligned}
		J_{\sigma}(u)=&\mathop{\arg\max}\limits_{v}v\Phi_{\sigma}(u-v)\\
		=&\mathop{\arg\max}\limits_{v}v\Phi_{1}(\frac{u-v}{\sigma})\\
		=&\mathop{\arg\max}\limits_{\omega=\frac{v}{\sigma}}\sigma\omega\Phi_1(\frac{u}{\sigma}-\omega)\\
		=&\sigma\mathop{\arg\max}\limits_{\omega}\Phi_{1}(\frac{u}{\sigma}-\omega)\\
		=&\sigma J_1(\frac{u}{\sigma}).
		\end{aligned}
		\end{equation*}
		When $\sigma\in(\frac{1}{2},1)$, we know $\frac{u^{*}}{\sigma}>u^{*}$. Since $J_{1}(u^{*})=u^{*}$ and that $u\in(u^{*},+\infty)\Leftrightarrow J_{1}(u)<u$, we have $\frac{u^{*}}{\sigma}>J_1(\frac{u^{*}}{\sigma})$. Hence
		\begin{equation}
		u^{*}>\sigma J_{1}(\frac{u^{*}}{\sigma})=J_{\sigma}(u^{*}).
		\label{J_and_u_under_small_sigma}
		\end{equation}
		
	\end{proof}
	
	Therefore, without losing generality, we assume that for the problem parameterized by $\sigma_2$, the price $v\in(0,u^{*})$. To be specific, suppose $v^{*}(\sigma)=J_{\sigma}(u^{*})$. Define $\Psi_{t+1}: [0,1]^t\rightarrow(0,u^{*})$ as any policy that proposes a price at time $t+1$. Define $\Psi=\{\Psi_{1}, \Psi_{2}, \ldots, \Psi_{T-1},\Psi_{T}\}$.
	
	Define the sequence of price as $V=\{v_1, v_2, \ldots, v_{T-1}, v_T \}$, and the sequence of decisions as $\mathbbm{1}=\{\mathbbm{1}_{1}, \mathbbm{1}_{2}, \ldots, \mathbbm{1}_{T-1}, \mathbbm{1}_{T}\}$. Denote $V^{t}=\{v_1, v_2, \ldots, v_t, \}$.% and $\mathbbm{1}^{t}=\{\mathbbm{1}_{1}, \mathbbm{1}_{2}, \ldots, \mathbbm{1}_{t-1}, \mathbbm{1}_{t}\}$ as the record of prices and acceptance.
	
	Define the probability (also the likelihood if we change $u^{*}$ to other parameter $u$):
	\begin{equation}
	Q^{V,\sigma}_T(\mathbbm{1})=\prod_{t=1}^{T}\Phi_{\sigma}(u^{*}-v_t)^{\mathbbm{1}_t}\Phi_{\sigma}(v_t-u^{*})^{1-\mathbbm{1}_t }.
	\end{equation}
	
	Define a random variable $Y_t\in\{0,1\}^t, Y_t\sim Q^{V^t, \sigma}_{t}$ and one possible assignment \\
	$y_t = \{\mathbbm{1}_{1}, \mathbbm{1}_{2}, \ldots, \mathbbm{1}_{t-1}, \mathbbm{1}_{t}\}$ . For any price $v$ and any parameter $\sigma$, define the expected reward function as $r(v,\sigma):=v\Phi_{\sigma}(u^{*}-v)$. Based on this, we can further define the expected regret $\mathrm{Regret}(\sigma, T, \Psi)$:
	
	\begin{equation}
	\mathrm{Regret}(\sigma, T, \Psi) = \E[\sum_{t=1}^{T}r(J_{\sigma}(u^{*}),\sigma)-r(\Psi_{t}(y_{t-1}),\sigma)]
	\end{equation}
	
	Now we have the following properties:

\begin{lemma}
	\begin{enumerate}
		\item $r(v^{*}(\sigma), \sigma)-r(v, \sigma)\geq\frac{1}{60}(v^{*}(\sigma)-v)^2$;
		\item $|v^{*}(\sigma)-u^{*}|\geq\frac{2}{5}|1-\sigma|$;
		\item $|\Phi_{\sigma}(u^{*}-v)-\Phi_{1}(u^{*}-v)|\leq|u^{*}-v|\cdot |\sigma - 1 |$.
	\end{enumerate}
	\label{lemma_properties}
\end{lemma}

\begin{proof}
	\begin{enumerate}
		\item We have:
		\begin{equation*}
		\begin{aligned}
		\frac{\partial r(v,\sigma)}{\partial v}|_{v=v^{*}(\sigma)}&=0
		\frac{\partial^2r(v,\sigma)}{\partial v^2} &= \frac{1}{\sigma^2}(v^2-u^{*}v-2\sigma^2)p_{\sigma}(u^{*}-v)
		\end{aligned}
		\end{equation*}
		Since $v\in(0,u^{*})$, we have $(v^2-u^{*}v-2\sigma^2)<-2\sigma^2$. Also, since $\sigma\in (1/2,1)$, we have $p_{\sigma}(u^{*}-v)>\frac{1}{\sqrt{2\pi}}\cdot e^{-\frac{(u^{*})^{2}}{2\cdot(1/2)^2}}=\frac{1}{\sqrt{2\pi}e^{\pi}}>0.017$. Therefore, we have
		\begin{equation*}
		\frac{\partial^2r(v,\sigma)}{\partial v^2}<-2*0.017<-\frac{1}{30}
		\end{equation*}
		As a result, we have:
		\begin{equation}
		\begin{aligned}
		r(v*(\sigma),\sigma)-r(v,\sigma)&=-(v*(\sigma)-v)\frac{\partial r(v,\sigma)}{\partial v}|_{v=v^{*}(\sigma)}-\frac{1}{2}(v*(\sigma)-v)^2\frac{\partial^2r(v,\sigma)}{\partial v^2}|_{v=\tilde{v}}\\
		&=0-\frac{1}{2}(v*(\sigma)-v)^2\frac{\partial^2r(v,\sigma)}{\partial v^2}|_{v=\tilde{v}}\\
		&\geq\frac{1}{2}\cdot\frac{1}{30}(v*(\sigma)-v)^2.
		\end{aligned}
		\end{equation}
		\item According to Equation \ref{J_and_u_under_small_sigma}, we know that:
		\begin{equation*}
		v^{*}(\sigma)=\sigma J_{1}(\frac{u^{*}}{\sigma})
		\end{equation*}
		For $u\in (u^{*},+\infty), J_1(u)<u$. According to Lemma \ref{lemma_j_prime}, we have:
		\begin{equation*}
		\begin{aligned}
		J'_{1}(u) &= 1+ \frac{1}{J_1(u)(J_1(u)-u)-2}\\
		&>1+\frac{1}{0-2}\\
		&=\frac12.
		\end{aligned}
		\end{equation*}
		Also, for $u\in(u^{*},\frac{u^{*}}{\sigma})$, we have:
		\begin{equation*}
		\begin{aligned}
		J'_{1}(u) &= 1- \frac{1}{2+J_1(u)(u-J_1(u))}\\
		&\explain{\leq}{0<J_1(u)<u} 1-\frac{1}{2+u(u-0)}\\
		&\explain{\leq}{u<\frac{u^{*}}{\sigma}<\frac{u^{*}}{2}} 1-\frac{1}{2+(\frac{u^{*}}{2})^2}\\
		&=1-\frac{1}{2+\frac{\pi}{8}}\\
		&<\frac{3}{5}.
		\end{aligned}
		\end{equation*}
		Therefore, we have:
		\begin{equation*}
		\begin{aligned}
		J_{1}(u^{*})-J_{\sigma}(u^{*})&=J_1(u^{*})-\sigma J_{1}(\frac{u^{*}}{\sigma})\\
		&=J_{1}(u^{*})-\sigma J_{1}(u^{*})+\sigma(J_{1}(u^{*})-J_{1}(\frac{u^{*}}{\sigma}))\\
		&=J_{1}(u^{*})(1-\sigma) -\sigma(J_{1}(\frac{u^{*}}{\sigma})-J_{1}(u^{*}))\\
		&>u^{*}(1-\sigma)-\sigma\cdot\frac{3}{5}(\frac{u^{*}}{\sigma}-u^{*})\\
		&=u^{*}(1-\sigma)-\frac35\sigma(\frac{1}{\sigma}-1)u^{*}\\
		&=u^{*}(1-\sigma)(1-\frac35)\\
		&>\frac25(1-\sigma).
		\end{aligned}
		\end{equation*}
		\item This is because:
		\begin{equation*}
		\begin{aligned}
		|\Phi_{\sigma}(u^{*}-v)-\Phi_{1}(u^{*}-v)|&=|\Phi_{\sigma}(u^{*}-v)-\Phi_{\sigma}(\sigma(u^{*}-v))|\\
		&\leq\max|p_{\sigma}|\cdot|(u^{*}-v)-\sigma(u^{*}-v)|\\
		&\leq\frac{1}{\sqrt{2\pi}\sigma}\cdot|(u^{*}-v)-\sigma(u^{*}-v)|\\
		&\explain{\leq}{\sigma > \frac12}(1-\sigma)|u^{*}-v|.
		\end{aligned}
		\end{equation*}
	\end{enumerate}
\end{proof}

In the following part, we will propose two theorems, which balance the cost of learning and that of uncertainty. This part is mostly similar to [BR12] Section 3, but we adopt a different family of demand curves here.

\begin{theorem}[Learning is costly] Let $\sigma\in(1/2,1)$ and $v_t\in(0,u^{*})$, and we have: % Suppose $\sigma=1-T^{-\frac14}$, then we have:
	\begin{equation}
	\mathcal{K}(Q^{V,1};Q^{V,\sigma})< 9900 (1-\sigma)^2 \mathrm{Regret}(1, T, \Psi).
	\end{equation}
	Here $v_t = \Psi(y_{t-1}), t=1,2,\ldots, T$.
	\label{theorem_learning}
\end{theorem}
\begin{proof}
	First of all, we cite the following lemma that would facilitate the proof.
	\begin{lemma}[Corollary 3.1 in Taneja and Kumar, 2004]
		Suppose $B_1$ and $B_2$ are distributions of Bernoulli random variables with parameters $q_1$ and $q_2$, respectively, with $q_1, q_2\in (0, 1)$. Then,
		\begin{equation*}
		\mathcal{K}(B_1;B_2)\leq\frac{(q_1-q_2)^2}{q_2(1-q_2)}.
		\end{equation*}
		\label{lemma_Taneja_and_Kumar_bernoulli_divergence}
	\end{lemma}
	According to the definition of KL-divergence, we have:
	\begin{equation*}
	\mathcal{K}(Q_T^{V,1};Q_T^{V,\sigma})=\sum_{s=1}^{T}\mathcal{K}(Q_s^{V^s,1};Q_s^{V^s,\sigma}|Y_{s-1}).
	\end{equation*}
	For each term of the RHS, we have:
	\begin{equation*}
	\begin{aligned}
	&\qquad\mathcal{K}(Q_s^{V^s,1}, Q_s^{V^s,\sigma}|Y_{s-1})\\
	&=\quad\sum_{y_s\in\{0,1\}^{s}}Q_s^{V^s,1}(y_s)\log\left(\frac{Q_s^{V^s,1}(\mathbbm{1}_s|y_{s-1})}{Q_s^{V^s,\sigma}(\mathbbm{1}_s|y_{s-1})}\right)\\
	&\explain{=}{\text{split  $y_s$ as $y_{s-1}$ and $ind_s$}}\qquad\qquad \sum_{y_{s-1}\in\{0,1\}^{s-1}}Q_{s-1}^{V^{s-1},1}(y_{s-1})\cdot\sum_{\mathbbm{1}_{s}\in\{0,1\}}Q_s^{V^s,1}(\mathbbm{1}_{s}|y_{s-1})\log\left(\frac{Q_s^{V^s,1}(\mathbbm{1}_s|y_{s-1})}{Q_s^{V^s,\sigma}(\mathbbm{1}_s|y_{s-1})}\right)\\
	&=\qquad \sum_{y_{s-1}\in\{0,1\}^{s-1}}Q_{s-1}^{V^{s-1},1}(y_{s-1})\mathcal{K}\left(Q_s^{V^s,1}(\cdot|y_{s-1}), Q_s^{V^s,\sigma}(\cdot|y_{s-1})\right)\\
	&\explain{\leq}{\text{Lemma \ref{lemma_Taneja_and_Kumar_bernoulli_divergence}}} \qquad\sum_{y_{s-1}\in\{0,1\}^{s-1}}Q_{s-1}^{V^{s-1},1}(y_{s-1})\frac{(\Phi_1(u^{*}-v_s)-\Phi_{\sigma}(u^{*}-v_s))^2}{\Phi_{\sigma}(u^{*}-v_s)(1-\Phi_{\sigma}(u^{*}-v_s))}\\
	&=\quad\frac{1}{\Phi_{\sigma}(u^{*}-v_s)(1-\Phi_{\sigma}(u^{*}-v_s))}\sum_{y_{s-1}\in\{0,1\}^{s-1}}Q_{s-1}^{V^{s-1},1}(y_{s-1})(\Phi_1(u^{*}-v_s)-\Phi_{\sigma}(u^{*}-v_s))^2\\
	&\explain{\leq}{(**)}\qquad 165\cdot\sum_{y_{s-1}\in\{0,1\}^{s-1}}Q_{s-1}^{V^{s-1},1}(y_{s-1})(\Phi_1(u^{*}-v_s)-\Phi_{\sigma}(u^{*}-v_s))^2\\
    &\explain{\leq}{\text{Lemma \ref{lemma_properties} Property 3}}\qquad165\cdot\sum_{y_{s-1}\in\{0,1\}^{s-1}}Q_{s-1}^{V^{s-1},1}(y_{s-1})(u^{*}-v_s)^2(1-\sigma)^2\\
	&=\qquad 165(1-\sigma)^2\E_{Y_{s-1}}[(u^{*}-v_s)^2].
	\end{aligned}
	\end{equation*}
	Here inequality (**) above is proved as follows: since $v_s\in(0,u^{*})$ as is assumed, we have:
	\begin{equation*}
	\begin{aligned}
	\frac12<\Phi_{\sigma}(u^{*}-v_s)<&\Phi_{\sigma}(u^{*})\\
	=&\sigma\cdot\Phi_{1}(\frac{u^{*}}{\sigma})\\
	\leq&1\cdot\Phi_1(\frac{\sqrt{\frac{\pi}{2}}}{\frac12})\\
	\leq&\Phi_{1}(\sqrt{2\pi})\\
	\leq&0.9939\ .
	\end{aligned}
	\end{equation*}
	As a result, we have $\frac{1}{\Phi_{\sigma}(u^{*}-v_s)(1-\Phi_{\sigma}(u^{*}-v_s))}\leq\frac{1}{0.9939\times0.0061}=164.7988\leq165$.
	Therefore, by summing up all $s$, we have:
	\begin{equation*}
	\begin{aligned}
	\mathcal{K}(Q_T^{V,1};Q_T^{V,\sigma})&=\sum_{s=1}^{T}\mathcal{K}(Q_s^{V^s,1};Q_s^{V^s,\sigma}|Y_{s-1})\\
	&\leq 165(1-\sigma)^2\sum_{s=1}^{T}\E_{Y_{s-1}}[(u^{*}-v_s)^2]\\
	&\explain{\leq}{\text{Lemma~\ref{lemma_properties} Property 1}} 165\times60\cdot(1-\sigma)^2\sum_{s=1}^{T}\left(r(u^{*},1)-r(v_s,1)\right)\\
	&\explain{=}{\text{definition of regret and } v_s=\Psi(y_{s-1}).} 9900(1-\sigma)^2\mathrm{Regret}(1,T,\Psi),
	\end{aligned}
	\end{equation*}
which concludes the proof. %\yw{Propagate the formatting changes in ``explainup'' to other places in the proof.}
\end{proof}
\begin{theorem}[Uncertainty is costly] Let $\sigma\leq1-T^{-\frac14}$, and we have:% Let $\sigma\in(1/2,1)$ and $v_t\in(0,u^{*})$, and we have:% 
	\begin{equation}
	\mathrm{Regret}(1, T, \Psi)+\mathrm{Regret}(\sigma, T, \Psi) \geq \frac{1}{24000}\cdot \sqrt{T}\cdot e^{-\mathcal{K}(Q^{V,1};Q^{V,\sigma})}.
	\end{equation}
	Here $v_t = \Psi(y_{t-1}), t=1,2,\ldots, T$.
	\label{theorem_uncertainty}
\end{theorem}
\begin{proof}
	First of all, we cite a lemma that would facilitate our proof:
	\begin{lemma}
		\label{lemma_Tsybakov2009}
		Let $Q_0$ and $Q_1$ be two probability distributions on a finite space $\mathcal{Y}$; with $Q_0(y), Q_1(y) > 0, \forall y\in\mathcal{Y}$. Then for any function $F: \mathcal{Y}\rightarrow\{0,1\}$,
		\begin{equation*}
		Q_0\{F=1\}+Q_1\{J=0\}\geq\frac{1}{2} e^{-\mathcal{K}(Q_0;Q_1)},
		\end{equation*}
		where $\mathcal{K}(Q_0;Q_1)$ denotes the KL-divergence of $Q_0$ and $Q_1$.
	\end{lemma}
	Define two intervals of prices:
	\begin{equation*}
	C_1=\{v:|u^{*}|\leq\frac{1}{10T^{\frac14}}\}\ \  and\ \ C_2=\{v: |J_{\sigma}(u^{*})-v|\leq\frac{1}{10T^{\frac14}}\}
	\end{equation*}
	Note that $C_1$ and $C_2$ are disjoint, since $|u^{*}-J_{\sigma}(u^{*})|\geq\frac25|1-\sigma|=\frac{2}{5T^{1/2}}$ according to Lemma \ref{lemma_properties} Property 2. Also, for $v\in (0,u^{*})\backslash C_2$, the regret is large according to Lemma \ref{lemma_properties} Property 1, because:
	\begin{equation*}
	r(v^{*}(\sigma),\sigma)-r(v,\sigma)\geq\frac{1}{60}(v-v^{*}(\sigma))^2\geq\frac{1}{6000{T}^{\frac12}}.
	\end{equation*}
	Then, we have:
	\begin{equation*}
	\begin{aligned}
	&\mathrm{Regret}(1, T, \Psi)+\mathrm{Regret}(\sigma, T, \Psi)\\
	\geq&\ \sum_{t=1}^{T-1}\E_1[r(u^{*},1)-r(v_{t+1},1)]+\E_{\sigma}[r(J_{\sigma}(u^{*}),\sigma)-r(v_{t+1},\sigma)]\\
	\geq&\ \frac{1}{6000\sqrt{T}}\sum_{t=1}^{T-1}\P_1[v_{t+1}\notin{C_1}]+\P_{\sigma}[v_{t+1}\notin\{C_2\}]\\
	\explain{\geq}{Suppose\ F_{t+1}=\mathbbm{1}[v_{t+1}\in C_2]}&\ \frac{1}{6000\sqrt{T}}\sum_{t=1}^{T-1}\P_1[F_{t+1}=1]+\P_{\sigma}[F_{t+1}=0]\\
	\explain{\geq}{Lemma\ \ref{lemma_Tsybakov2009}}&\ \frac{1}{6000\sqrt{T}}\sum_{t=1}^{T-1}\frac{1}{2} e^{-\mathcal{K}(Q_t^{V^t,1};Q_t^{V^t,\sigma})}\\
	\explain{\geq}{\mathcal{K}(Q_t^{V^t,1};Q_t^{V^t,\sigma})\ not\ decreasing\ }&\ \frac{1}{6000\sqrt{T}}\frac{T-1}{2}e^{-\mathcal{K}(Q_T^{V,1};Q_T^{V,\sigma})}\\
	\geq&\ \frac{1}{24000}\sqrt{T}e^{-\mathcal{K}(Q_T^{V,1};Q_T^{V,\sigma})}.
	\end{aligned}
	\end{equation*}
\end{proof}

According to Theorem \ref{theorem_learning} and Theorem \ref{theorem_uncertainty}, we can then prove Theorem \ref{theorem_lower_bound_square_root_t}. Let $\sigma=1-T^{-\frac{1}{4}}$
\begin{equation*}
\begin{aligned}
&2\left(\mathrm{Regret}(1, T, \Psi)+\mathrm{Regret}(\sigma, T, \Psi)\right)\\
\geq&\mathrm{Regret}(1, T, \Psi)+\left(\mathrm{Regret}(1, T, \Psi)+\mathrm{Regret}(\sigma, T, \Psi)\right)\\
\geq&\frac{1}{9900T^{-1/2}}\mathcal{K}(Q^{V,1};Q^{V,\sigma})+\frac{1}{24000}\cdot \sqrt{T}\cdot e^{-\mathcal{K}(Q^{V,1};Q^{V,\sigma})}\\
\geq&\frac{1}{24000}\sqrt{T}\left(\mathcal{K}(Q^{V,1};Q^{V,\sigma})+e^{-\mathcal{K}(Q^{V,1};Q^{V,\sigma})}\right)\\
\explain{\geq}{The\ fact\ e^{x}\geq x+1,\forall{x}\in\R }&\frac{1}{24000}\sqrt{T}.
\end{aligned}
\end{equation*}
Thus Theorem \ref{theorem_lower_bound_square_root_t} is proved valid. 
\end{proof}

\section{More Discussions} \label{appendix_sec_more_discussion}
\subsection{Dependence on \texorpdfstring{$B$}{Lg} and Noise Variance}\label{appendix_coefficient}
Here we use a concrete example to analyze the coefficients of regret bounds. Again, we assume that $N_t\sim\cN(0,\sigma^2)$. Notice that both $C_{s}$ and $C_a$ have a component of $\frac{C_{\exp}}{C_{down}}$. In order to analyze $\frac{C_{\exp}}{C_{down}}$, we define a \emph{hazard function} denoted as $\lambda(\omega)$ with $\omega\in\mathbb{R}$:

\begin{equation}
\lambda(\omega):=\frac{p_{1}(\omega)}{1-\Phi_{1}(\omega)}=\frac{p_{1}(-\omega)}{\Phi_{1}(-\omega)},
\label{hazard_function}
\end{equation}

where $\Phi_1$ and $p_1$ are the CDF and PDF of standard Gaussian distribution. The concept of hazard function comes from the area of \emph{survival analysis}. From Equation \ref{eq__C_down} and \ref{eqn_def_c_exp}, we plug in Equation \ref{hazard_function} and get:

\begin{equation}
\begin{aligned}
C_{\text{down}}&\geq\inf_{\omega\in[-\frac{B}{\sigma}, \frac{B}{\sigma}]}\left\{\frac{1}{\sigma^2}\lambda(-\omega)^2+\omega\cdot\lambda(-\omega)\right\}\\
C_{\text{exp}}&\leq\sup_{\omega\in[-\frac{B}{\sigma}, \frac{B}{\sigma}]}\left\{\frac1{\sigma^2}\lambda(-\omega)^2\right\}.
\end{aligned}
\label{equ_down_exp}
\end{equation}

In Lemma \ref{lemma_lambda}, we will prove that $\lambda(\omega)$ is exponentially small as $\omega\rightarrow+\infty$, and is asymptotically close to $-\omega$ as $\omega\rightarrow-\infty$. Therefore, $C_{down}$ is exponentially small and $C_{exp}$ is quadratically large with respect to $B/\sigma$.  Although we assume that $B$ and $\sigma$ are constant, we should be alert that the scale of $B/\sigma$  can be very large as $\sigma$ goes to zero, i.e. as the noise is ``insignificant''. In practice (especially when $T$ is finite), this may cause extremely large regret at the beginning. A ``Shallow Pricing'' method introduced by \citet{cohen2020feature_journal} (as well as other domain-cutting methods in contextual searching) may serve as a good pre-process as it frequently conducts bisections to cut the feasible region of $\theta^{*}$ with high probability. According to Theorem 3 in \citet{cohen2020feature_journal}, their Shallow Pricing algorithm will bisect the parameter set for at most logarithmic times to ensure that $\frac{B}{\sigma}$ has been small enough (i.e. upper-bounded by $O(poly\log(T))$). However, this does not necessarily means that we can use a $O(\log{T})$-time pre-process to achieve the same effect, since they run the algorithm throughout the session while we only take it as a pre-process. Intuitively, at least under the adversarial feature assumption, we cannot totally rely on a few features occurring at the beginning (as they might be misleading) to cut the parameter set once and for all. A mixture approach of Shallow Pricing and EMLP/ONSP might work, as the algorithm can detect whether current $\frac{B}{\sigma}$ is larger than a threshold of bisection. However, this requires new regret analysis as the operations parameter domain are changing over time. Therefore, we claim in Section \ref{sec_discussion} that the regret bound is still open if $\sigma=\Theta(T^{-\alpha})$ for $\alpha\in(0,1)$. % However, this is not necessarily mean that we can always successfully cut the set within a $O(\log(T))$-length pre-process, since these bisections might happen at any time in the total $T$ rounds if we run Shallow Pricing throughout the whole session.  %According to Theorem 3 in \citet{cohen2020feature}, if $\frac{B}{\sigma}=\Omega(T\log T)$, the pre-process will only last $O(\log T)$ times and therefore not intervene the main algorithm. %As long as $\frac{B}{\sigma}$ gets relatively smaller, we will reduce the regret coefficient as well.

\begin{lemma}[Properties of $\lambda(\omega)$] For $\lambda(\omega):=\frac{p_{1}(\omega)}{1-\Phi_{1}(\omega)}$, we have:
	\begin{equation*}
	\begin{aligned}
	1,&\qquad \frac{d}{d\omega}\lambda(\omega)>0.\\
	2,&\qquad \lim_{\omega\rightarrow-\infty}\omega^k\lambda(\omega)=0,\ \forall k>0.\\
	3,&\qquad \lim_{\omega\rightarrow+\infty}\lambda(\omega)-\omega=0.\\
	4,&\qquad \lim_{\omega\rightarrow+\infty}\omega\left(\lambda(\omega)-\omega\right)=1.\\
	\end{aligned}
	\end{equation*}

	\label{lemma_lambda}
\end{lemma}

\begin{proof}
			We prove the Lemma \ref{lemma_lambda} sequentially:
			\begin{enumerate}
				\item We have:
				\begin{equation}
				\begin{aligned}
				\lambda'(\omega)=&\frac{p_{1}^2(-\omega)-p'_{1}(-\omega)\Phi_{1}(-\omega)}{\Phi_{1}(-\omega)^2}\\
				=&\frac{p_{1}^2(-\omega)-\omega p_{1}(-\omega)\Phi_{1}(-\omega)}{\Phi_{1}(-\omega)^2}\\
				=&\frac{p_{1}(-\omega)\left(p_{1}(-\omega)-\omega\Phi_{1}(-\omega)\right)}{\Phi_{1}(-\omega)^2}.
				\end{aligned}
				\end{equation}
				Therefore, it is equivalent to prove that $p_{1}(-\omega)-\omega\Phi_{1}(-\omega)>0$.
				
				Suppose $f(\omega)=p_{1}(\omega)+\omega\Phi_{1}(\omega)$. We now take its derivatives as follows:
				\begin{equation}
				\begin{aligned}
				f'(\omega)&=p'_{1}(\omega)+(\Phi_{1}(\omega)+\omega\cdot p_{1}(\omega))\\
				&=(-\omega)p_{1}(\omega)+\Phi_{1}(\omega)+\omega\cdot p_{1}(\omega)\\
				&=\Phi_{1}(\omega)\\
				&>0
				\end{aligned}
				\label{f_prime_>_0}
				\end{equation}
				
				Therefore, we know that $f(\omega)$ monotonically increases in $\mathbb{R}$. Additionally, since we have:
				\begin{equation}
				\begin{aligned}
				&\lim_{\omega\rightarrow-\infty}f(\omega)\\
				=&\lim_{\omega\rightarrow-\infty}p_{1}(\omega)+\lim_{\omega\rightarrow-\infty}\omega\Phi_{1}(\omega)\\
				=&0+\lim_{\omega\rightarrow-\infty}\frac{1}{\sigma^2}\cdot\frac{\Phi_{1}(\omega)}{1/\omega}\\
				=&\lim_{\omega\rightarrow-\infty}\cdot\frac{p_{1}(\omega)}{-1/\omega^2}\\
				=&\lim_{\omega\rightarrow-\infty}\cdot\left(-\frac{1}{\sqrt{2\pi}}\cdot\frac{\omega^2}{\exp\{\frac{\omega^2}{2}\}}\right)\\
				=&0
				\end{aligned}
				\label{lim_f_>_0}
				\end{equation}
				
				Therefore, we know that $f(\omega)>0$, $\forall \omega\in\mathbb{R}$, and as a result, $\lambda'(\omega)>0$.
				\item We have:
				\begin{equation}\label{poly_omega_times_lambda_lim_0}
				\begin{aligned}
				&\lim_{\omega\rightarrow-\infty}\omega^k\lambda(\omega)\\
				=&\lim_{\omega\rightarrow-\infty}\omega^k\frac{p_{1}(-\omega)}{\Phi_{1}(-\omega)}\\
				=&\frac{\mathop{\lim}\limits_{\omega\rightarrow-\infty}\omega^kp_{1}(-\omega)}{\mathop{\lim}\limits_{\omega\rightarrow-\infty}\Phi_{1}(-\omega)}\\
				=&\frac{\mathop{\lim}\limits_{\omega\rightarrow-\infty}\omega^k(\frac{1}{\sqrt{2\pi}}\exp\{-\frac{\omega^2}{2}\})}{1}\\
				=&\frac{0}{1}\\
				=&0.\\
				\end{aligned}
				\end{equation}

				\item We only need to prove that
				\begin{equation*}
				\lim_{\omega\rightarrow+\infty}\lambda(\omega)-\omega=0.
				\end{equation*}
				Actually, we have:
				\begin{equation}
				\begin{aligned}
				&\lim_{\omega\rightarrow+\infty}\lambda(\omega)-\omega\\
				=&\lim_{\omega\rightarrow+\infty}\frac{p_{1}(-\omega)-\omega\Phi_{1}(-\omega)}{\Phi_{1}(-\omega)}\\
				= &\lim_{\omega\rightarrow-\infty}\frac{p_{1}(\omega)+\omega\Phi_{1}(\omega)}{\Phi_{1}(\omega)}\\
				\explainup{=}{\text{L'Hospital's rule}}&\lim_{\omega\rightarrow-\infty}\frac{(-\omega)p_{1}(\omega)+\Phi_{1}(\omega)+\omega p_{1}(\omega)}{p_{1}(\omega)}\\
				=&\lim_{\omega\rightarrow-\infty}\frac{\Phi_{1}(\omega)}{p_{1}(\omega)}\\
				=&\lim_{\omega\rightarrow-\infty}\frac{p_{1}(\omega)}{(-\omega)p_{1}(\omega)}\\
				=&0
				\end{aligned}
				\end{equation}
				\item 
				\begin{equation}
				\begin{aligned}
				&\lim_{\omega\rightarrow+\infty}\omega(\lambda(\omega)-\omega)\\
				=&\lim_{\omega\rightarrow+\infty}\frac{\omega\left(p_{1}(-\omega)-\omega\Phi_{1}(-\omega)\right)}{\Phi_{1}(-\omega)}\\
				=&\lim_{\omega\rightarrow-\infty}\frac{-\omega p_{1}(\omega)-\omega^2\Phi_{1}(\omega)}{\Phi_{1}(\omega)}\\
				\explainup{=}{\text{L'Hospital's rule}}&\lim_{\omega\rightarrow-\infty}\frac{-p_{1}(\omega)-\omega(-\omega)p_{1}(\omega)-\omega^2p_{1}(\omega)-2\omega\cdot\Phi_{1}(\omega)}{p_{1}(\omega)}\\
				=&-1-2\lim_{\omega\rightarrow-\infty}\frac{\omega\Phi_{1}(\omega)}{p_{1}(\omega)}\\
				=&-1+2\lim_{\omega\rightarrow+\infty}\frac{1}{\frac{\lambda(\omega)}{\omega}}\\
				=&-1+2\\
				=&1.
				\end{aligned}
				\end{equation}
				
			\end{enumerate}
			Thus the lemma holds.
			
		\end{proof}

\subsection{Algorithmic Design}
\subsubsection{Probit and Logistic Regressions} 
\label{appendix_probit_regression}

A probit/logit model is described as follows: a Boolean random variable $Y$ satisfies the following probabilistic distribution: $\P[Y=1|X]=F(X^{\top}\beta)$, where $X\in\mathbb{R}$ is a random vector, $\beta\in\mathbb{R}$ is a parameter, and $F$ is the cumulative distribution function (CDF) of a (standard) Gaussian/logistic distribution. In our problem, we may treat $\mathbbm{1}_{t}$ as $Y$, $[{x_{t}}^{\top}, v_{t}]^{\top}$ as $X$ and $[{\theta^{*}}^{\top}, -1]^{\top}$ as $\beta$, which exactly fits this model if we assume the noise as Gaussian or logistic. Therefore,  $\hat{\theta}_k=\mathop{\arg\min}_{\theta}\hat{L}_{k}(\theta)$ can be solved via the highly efficient implementation of generalized linear models, e.g., GLMnet, rather than resorting to generic tools for convex programming. As a heuristic, we could leverage the vast body of statistical work on probit or logit models and adopt a fully Bayesian approach that jointly estimates $\theta$ and hyper-parameters of $F$. This would make the algorithm more practical by eliminating the need to choose the hyper-parameters when running this algorithm.

\subsubsection{Advantages of EMLP over ONSP.}

For the stochastic setting, we specifically propose EMLP even though ONSP also works. This is because EMLP only ``switch'' the pricing policy $\hat{\theta}$ for $\log{T}$ times. This makes it appealing in many applications (especially for brick-and-mortar sales) where the number of policy updates is a bottleneck. In fact, the iterations within one epoch can be carried out entirely in parallel.

\subsubsection{Agnostic Dynamic Pricing: Explorations versus Exploitation}\label{appendix_expl_vs_expl}
At the moment, the proposed algorithm relies on the assumption of a linear valuation function (see Appendix \ref{appendix_subsec_problem_modeling} for more discussion on problem modeling). It will be interesting to investigate the settings of model-misspecified cases and the full agnostic settings. The key would be to exploit the structural feedback in model-free policy-evaluation methods such as importance sampling. The main reason why we do not explore lies in the noisy model: essentially we are implicitly exploring a higher (permitted) price using the naturally occurring noise in the data. In comparison, there is another problem setting named ``adversarial irrationality'' where some of the customers will valuate the product adaptively and adversarially\footnote{An adaptive adversary may take actions adversarially in respond to the environmental changes. In comparison, what we allow for the ``adversarial features'' is actually chosen by an oblivious adversary before the interactions start.}. Existing work \citet{krishnamurthy2020contextual} adopts this setting and shows a linear regret dependence on the number of irrational customers, but they consider a different loss function (See Related Works Section).

\subsection{Problem Modeling}\label{appendix_subsec_problem_modeling}

\subsubsection{Noise Distributions}\label{appendix_noise}

%In this work, we assume the noise a Gaussian. We have exploited its good properties, such as a symmetric and smoothly differentiable PDF, an exponentially decaying tail bound and an accessible value table of CDF (for proof). In comparison, the work \citet{javanmard2019dynamic} assumed a more general noise distributions: its CDF $F(v)$ strictly increases and $F(v)$ as well as $\left(1-F(v)\right)$ is log-concave. Since they also assume the continuous and smoothness of its PDF $f(v)$ by taking derivatives $f'(v)$ in the proof, we are still unaware of the regret bound if the noise distribution is discrete, or at least $f(v)$ is not differentiable?
In this work, we have made four assumptions on the noise distribution: strict log-concavity, $2^{\text{nd}}-\text{order smooth}$, known, and i.i.d.. Here we explain each of them specifically.
\begin{itemize}
	\item The assumption of knowing the exact $F$ is critical to the regret bound: If we have this knowledge, then we achieve $O(\log{T})$ even with adversarial features; otherwise, an $\Omega(\sqrt{T})$ regret is unavoidable even with stochastic features.
	\item The strictly log-concave distribution family includes Gaussian and logistic distributions as two common noises. In comparison, \citet{javanmard2019dynamic} assumes log-concavity that further covers Laplacian, exponential and uniform distributions. \citet{javanmard2019dynamic} also considers the cases when (1) the noise distribution is unknown but log-concave, and (2) the noise distribution is zero-mean and bounded by support of $[-\delta, \delta]$. For case (1), they propose an algorithm with regret $O(\sqrt{T})$ and meanwhile prove the same lower bound. For case (2), they propose an algorithm with linear regret.
	\item The assumption that $F$ is $2^{\text{nd}}-$order smooth is also assumed by \citet{javanmard2019dynamic} by taking derivatives $f'(v)$ and applying its upper bound in the proof. Therefore, we are still unaware of the regret bound if the noise distribution is discrete, where a lower bound of $\Omega(\sqrt{T})$ can be directly applied from \citet{kleinberg2003value}.
	\item We even assume that the noise is identically distributed. However, the noise would vary among different people. The same problem happens on the parameter $\theta^{*}$: can we assume different people sharing the same evaluation parameter? We may interpret it in the following two ways, but there are still flaws: (1) the ``customer'' can be the public, i.e. their performance is quite stable in general; or (2) the customer can be the same one over the whole time series. However, the former explanation cannot match the assumption that we just sell one product at each time, and the latter one would definitely undermine the independent assumption of the noise: people would do ``human learning'' and might gradually reduce their noise of making decisions. To this extent, it is closer to the fact if we assume noises as martingales. This assumption has been stated in \citet{qiang2016dynamic}. 
	
\end{itemize}

\subsubsection{Linear Valuations on Features}\label{appendix_linear}

There exist many products whose prices are not linearly dependent on features. One famous instance is a diamond: a kilogram of diamond powder is very cheap because it can be produced artificially, but a single 5-carat (or 1 gram) diamond might cost more than \$100,000. This is because of an intrinsic non-linear property of diamond: large ones are rare and cannot be (at least easily) compound from smaller ones. Another example lies in electricity pricing \citep{joskow2012dynamic}, where the more you consume, the higher unit price you suffer. On the contrary, commodities tend to be cheaper than retail prices. These are both consequences of marginal costs: a large volume consuming of electricity may cause extra maintenance and increase the cost, and a large amount of purchasing would release the storage and thus reduce their costs. In a word, our problem setting might not be suitable for those large-enough features, and thus an upper bound of $x^{\top}\theta$ becomes a necessity.

\subsection{\emph{Ex Ante} v.s. \emph{Ex Post} Regrets}
In this work, we considered the \emph{ex ante} regret $Reg_{ea} = \sum_{t=1}^{T}\max_{\theta}\E[v^{\theta}_t\cdot\ind(v^{\theta}_t \leq w_t)] - \E[v_t\cdot\ind(v_t \leq w_t)]$, where $v^{\theta}_t = J(x_t^{\top}\theta)$ is the greedy price with parameter $\theta$ and $w_t = x_t^{\top}\theta^{*}+N_t$ is the realized random valuation. The \emph{ex post} definition of the cumulative regret, i.e., $Reg_{ep} = \max_{\theta}\sum_{t=1}^{T}v^{\theta}_t\ind(v^{\theta}_t\leq w_t) - v_t\ind(v_t\leq w_t)$ makes sense, too. Note that we can decompose $\E\left[Reg_{ep}\right] =  Reg_{ea}  +  \E[ \max_{\theta}\sum_{t=1}^{T}v^{\theta}_t\ind(v^{\theta}_t\leq w_t)  - \sum_{t=1}^{T}v^{\theta^*}_t\ind(v^{\theta^*}_t\leq w_t) ].$ While it might be the case that the second term is $\Omega(\sqrt{dT})$ as the reviewer pointed out, it is a constant independent of the algorithm. For this reason, we believe using $Reg_{ea}$ is without loss of generality, and it reveals more nuanced performance differences of different algorithms.
%We may consider the $Reg_{ep}$ regret as a ``training loss’’ with $T$ training data, and $Reg_{ea}$ as a ``testing loss’’. From [Vapnik00] we know that $Reg_{ea} - Reg_{ep} = \tilde{O}(\sqrt{dT})$ holds for the same process with high probability. However, it is still open whether $Reg_{ep}$ could achieve $O(d\log{T})$. 

For an \emph{ex post} \emph{dynamic} regret, i.e., $Reg_{d} = \sum_{t=1}^{T}w_t - v_t\cdot\ind(v_t \leq w_t)$, it is argued in \citet{cohen2020feature_journal} that any policy must suffer an expected regret of $\Omega(T)$ (even if $\theta^{*}$ is known). We may also present a good example lies in $N_t\sim\cN(0,1), x_t^{\top}\theta^{*}=\sqrt{\frac{\pi}{2}}$ where the optimal price is $\sqrt{\frac{\pi}{2}}$ as well but the probability of acceptance is only 1/2, and this leads to a constant \emph{per-step}  regret of $\frac{1}{2}\sqrt{\frac{\pi}{2}}$.

\subsection{Ethic Issues}\label{appendix_ethic}

A field of study lies in ``personalized dynamic pricing'' \citep{aydin2009personalized, chen2018primal}, where a firm makes use of information of individual customers and sets a unique price for each of them. This has been frequently applied in airline pricing \citep{kramer2018airline}. However, this causes first-order pricing discrimination. Even though this ``discrimination'' is not necessarily immoral, it must be embarrassing if we are witted proposing the same product with different prices towards different customers. For example, if we know the coming customer is rich enough and is not as sensitive towards a price (e.g., he/she has a variance larger than other customers), then we are probably raising the price without being too risky. Or if the customer is used to purchase goods from ours, then he or she might have a higher expectation on our products (e.g., he/she has a $\theta=a\theta^{*}, a>1$), and we might take advantage and propose a higher price than others. These cases would not happen in an auction-based situation (such as a live sale), but might frequently happen in a more secret place, for instance, a customized travel plan.
\end{appendices}

\end{document}